%% file: Paper.tex
\documentclass{article}
\usepackage{microtype}
\usepackage{graphicx}
\usepackage{subfigure}
\usepackage{booktabs}
\usepackage[accepted,nohyperref]{icml2022}

\usepackage{amsmath}
\usepackage{amssymb}
\usepackage{amsthm}
\usepackage{bbm}
\usepackage{bm}
\usepackage{color}
\usepackage{dirtytalk}
\usepackage{dsfont}
\usepackage{enumerate}
\usepackage{listings}
\usepackage{mathtools}
\usepackage{natbib}
\usepackage{times}
\usepackage{xspace}

\usepackage[bookmarks=false]{hyperref}
\hypersetup{
  pdftex,
  pdffitwindow=true,
  pdfstartview={FitH},
  pdfnewwindow=true,
  colorlinks,
  linktocpage=true,
  linkcolor=Green,
  urlcolor=Green,
  citecolor=Green
}
\usepackage[capitalize,noabbrev]{cleveref}

\usepackage[textsize=tiny]{todonotes}

\newcommand{\commentout}[1]{}
\newcommand{\junk}[1]{}

\usepackage{thmtools}
\declaretheorem[name=Theorem,refname={Theorem,Theorems},Refname={Theorem,Theorems}]{theorem}
\declaretheorem[name=Lemma,refname={Lemma,Lemmas},Refname={Lemma,Lemmas},sibling=theorem]{lemma}

\newcommand{\cA}{\mathcal{A}}

\newcommand{\cG}{\mathcal{G}}

\newcommand{\cL}{\mathcal{L}}

\newcommand{\cN}{\mathcal{N}}

\newcommand{\cS}{\mathcal{S}}
\newcommand{\cT}{\mathcal{T}}

\newcommand{\cV}{\mathcal{V}}

\newcommand{\cX}{\mathcal{X}}

\newcommand{\realset}{\mathbb{R}}

\newcommand{\parents}{\mathsf{pa}}
\newcommand{\children}{\mathsf{ch}}

\newcommand{\E}[2]{\mathbb{E}_{#1} \left[#2\right]}
\newcommand{\condE}[2]{\mathbb{E} \left[#1 \,\middle|\, #2\right]}

\newcommand{\prob}[1]{\mathbb{P} \left(#1\right)}
\newcommand{\condprob}[2]{\mathbb{P} \left(#1 \,\middle|\, #2\right)}

\newcommand{\condvar}[2]{\mathrm{var} \left[#1 \,\middle|\, #2\right]}

\newcommand{\abs}[1]{\left|#1\right|}

\newcommand*\dif{\mathop{}\!\mathrm{d}}

\newcommand{\I}[1]{\mathds{1} \! \left\{#1\right\}}

\newcommand{\maxnorm}[1]{\|#1\|_\infty}

\newcommand{\normw}[2]{\|#1\|_{#2}}

\newcommand{\set}[1]{\left\{#1\right\}}

\newcommand{\T}{^\top}

\DeclareMathOperator*{\argmax}{arg\,max\,}

\mathchardef\mhyphen="2D

\def\Regret{\mathcal{R}}
\def\Bregret{\mathcal{BR}}

\newcommand{\hierts}{\ensuremath{\tt HierTS}\xspace}
\newcommand{\hiertstwo}{\ensuremath{\tt FlatTS}\xspace}
\newcommand{\ts}{\ensuremath{\tt TS}\xspace}

\icmltitlerunning{Deep Hierarchy in Bandits}

\begin{document}

\twocolumn[
\icmltitle{Deep Hierarchy in Bandits}

\icmlsetsymbol{equal}{*}

\begin{icmlauthorlist}
\icmlauthor{Joey Hong}{b}
\icmlauthor{Branislav Kveton}{a}
\icmlauthor{Sumeet Katariya}{a}
\icmlauthor{Manzil Zaheer}{dm}
\icmlauthor{Mohammad Ghavamzadeh}{gr}
\end{icmlauthorlist}

\icmlaffiliation{a}{Amazon}
\icmlaffiliation{b}{University of California, Berkeley}
\icmlaffiliation{dm}{DeepMind}
\icmlaffiliation{gr}{Google Research}

\icmlcorrespondingauthor{Joey Hong}{joey\_hong@berkeley.edu}

\vskip 0.3in
]

\printAffiliationsAndNotice{}

\begin{abstract}
Mean rewards of actions are often correlated. The form of these correlations may be complex and unknown a priori, such as the preferences of a user for recommended products and their categories. To maximize statistical efficiency, it is important to leverage these correlations when learning. We formulate a bandit variant of this problem where the correlations of mean action rewards are represented by a \emph{hierarchical Bayesian model} with latent variables. Since the hierarchy can have multiple layers, we call it \emph{deep}. We propose a hierarchical Thompson sampling algorithm (\hierts) for this problem, and show how to implement it efficiently for Gaussian hierarchies. The efficient implementation is possible due to a novel exact hierarchical representation of the posterior, which itself is of independent interest. We use this exact posterior to analyze the Bayes regret of \hierts in Gaussian bandits. Our analysis reflects the structure of the problem, that the regret decreases with the prior width, and also shows that hierarchies reduce the regret by non-constant factors in the number of actions. We confirm these theoretical findings empirically, in both synthetic and real-world experiments.
\end{abstract}

\input{Introduction}

\input{Setting}

\input{Algorithm}

\input{Posteriors}

\input{Analysis}

\input{Experiments}

\input{RelatedWork}

\input{Conclusions}

\bibliographystyle{icml2022}
\bibliography{References}

\clearpage
\onecolumn
\appendix

\input{AppendixA}

\clearpage

\input{AppendixB}

\end{document}

%% file: Introduction.tex
\section{Introduction}
\label{sec:introduction}

A \emph{contextual bandit} \citep{li10contextual,chu11contextual} is a sequential decision-making  problem where a \emph{learning agent} sequentially interacts with an environment over $n$ rounds. In each round, the agent observes a \emph{context}, chooses one of $K$ possible \emph{actions}, and then receives a \emph{reward} for the taken action. The agent aims to maximize its expected cumulative reward over $n$ rounds. It does not know the mean rewards of the actions \emph{a priori} and learns them by taking the actions. This forces the agent to choose between \emph{exploring} actions to learn about them and \emph{exploiting} the action with the highest estimated reward. As an example, in online shopping, the context can be a user's query, the actions are products, and the reward is a purchase indicator \cite{yue11linear,li16contextual}.

In many practical problems, the action space is large and cannot be explored naively. It is also not immediately obvious what a good generalization over the actions would be. However, the mean rewards of the actions are correlated, which presents an opportunity for more statistically-efficient exploration. For instance, in online shopping, many products are semantically similar and can be organized into a hierarchy: both a keyboard and monitor are computer accessories; and both computer accessories and home theatre systems are electronic devices. However, such hierarchies are not easy to represent in traditional bandit algorithms \citep{auer02finitetime,chapelle11empirical,kawale15efficient,sen17contextual}. Another example is classification with a bandit feedback, where the labels are clustered: the car and truck are vehicles, while the monkey and tiger are animals. We experiment with such a problem in \cref{sec:image classification}.

To address this issue, we study a bandit problem with a \emph{deep hierarchical structure} in the action space. This structure is represented using a \emph{hierarchical Bayesian model} \citep{lindley1972bayes,zhang17survey}, where each action is a leaf node in the associated tree. Each node has a \emph{node parameter}, and nodes with the same parent have their parameters drawn i.i.d.\ from a distribution parameterized by their parent's parameter. Because of this, exploration of one action teaches the agent about other actions, depending on how far they are in the tree.

We make the following contributions. First, we formalize the hierarchical Bayesian model $\cT$ of our environment. Second, we propose a \emph{Thompson sampling (TS)} algorithm on $\cT$ and call it \hierts. The main novelty in \hierts is a factorization of the posterior along $\cT$, which permits \emph{exact posterior sampling} and \emph{computationally-efficient updates}. A closed-form solution exists in Gaussian bandits and contextual linear bandits with Gaussian rewards. We derive a Bayes regret bound for \hierts that reflects the structure of $\cT$ and the impact of priors. The bound improves upon vanilla TS in a polynomial factor in the number of actions, and thus shows increased statistical efficiency due to the hierarchy. We validate these findings empirically and also apply \hierts to a challenging classification problem with label hierarchy.

%% file: Setting.tex
\section{Setting}
\label{sec:setting}

We use the following notation. Random variables are capitalized. For any positive integer $n$, we denote by $[n]$ the set $\{1, 2, \hdots, n\}$. We let $\I{\cdot}$ be the indicator function. The $i$-th entry of vector $v$ is $v_i$; unless the vector $v_j$ is already indexed, in which case we write $v_{j, i}$.

We consider a learning agent that interacts with a contextual bandit over $n$ rounds \citep{li10contextual,chu11contextual}. In round $t \in [n]$, the agent observes \emph{context} $X_t \in \cX \subseteq \realset^d$, takes an \emph{action} $A_t$ from an \emph{action set} $\cA$ of size $K$, and then observes a \emph{stochastic reward} $Y_t = r(X_t, A_t) + \epsilon$, where $r: \realset^d \times \cA \to \realset$ is a \emph{reward function} and $\epsilon$ is independent $\sigma^2$-sub-Gaussian noise.

Our problem is structured. In particular, the action set $\cA$ progressively breaks into finer clusters of actions with similar rewards. This decomposition is induced by a tree $\cT$ (\cref{fig:hierarchy}) over nodes $\cV \subset \mathbb{Z}$. Each \emph{leaf} of $\cT$ corresponds to an action $a \in \cA$ and we call it an \emph{action node}. Each \emph{internal node} of $\cT$ has at least $2$ and at most $b$ children, where $b$ is the \emph{branching factor}. The \emph{height of the tree} is $h$ and the \emph{height of node} $i \in \cV$ is $h_i \in [0, h]$. The height of the leaves is $0$ and that of the root is $h$. Without loss of generality, the root has index $1$. For any node $i \in \cV$, we denote its parent by $\parents(i)$ and its children by $\children(i)$. An \emph{ancestor} of node $i$ is any node on a direct path from node $i$ to the root, and node $i$ is its \emph{descendant}. With a slight abuse of notation, we use $\cA \subseteq \cV$ to refer to both the action set and leaves of $\cT$, and sometimes index action nodes by $a$ to stress their role.

The reward function is parameterized by \emph{model parameters} $\Theta = (\theta_i)_{i \in \cV}$, where $\theta_i$ is the parameter of node $i$. The \emph{true model parameters} are $\Theta_* = (\theta_{*, i})_{i \in \cV}$ and we assume that they are generated as
\begin{align}
  \theta_{*, 1}
  & \sim P_{0, 1}\,,
  \label{eq:hierarchy} \\
  \theta_{*, i} \mid \theta_{*, \parents(i)}
  & \sim P_{0, i}(\cdot \mid \theta_{*, \parents(i)})\,, 
  & \forall i \in \cV \setminus \{1\}\,,
  \nonumber \\
  Y_t \mid X_t, \theta_{*, A_t} 
  & \sim P(\cdot \mid X_t; \theta_{*, A_t})\,,
  & \forall t \in [n]\,.
  \nonumber
\end{align}
Here $P_{0, 1}$ is the prior distribution of the root node, which we call a \emph{hyper-prior}, $P_{0, i}(\cdot \mid \theta_{*, i})$ is the \emph{conditional prior distribution} of node $i$, which is parameterized by the sampled value of its parent $\theta_{*, \parents(i)}$, and $P(\cdot \mid x; \theta_{*, a})$ is the \emph{reward distribution} of action $a$ in context $x$. We denote the mean reward of action $a$ in context $x$ under model parameters $\Theta$ by $r(x, a; \Theta)$, and relate it to the reward distribution as $\E{Y \sim P(\cdot \mid x; \theta_a)}{Y} = r(x, a; \Theta)$. Thus $r(x, a; \Theta)$ depends only on one parameter in $\Theta$. The generative process in \eqref{eq:hierarchy} relates any two node parameters to each other, through the lowest common ancestor. This induces complex correlations that can be used for efficient exploration. We discuss motivating examples for this setting in \cref{sec:introduction}.

The goal is to minimize the $n$-round regret defined as
\begin{align*}
  \Regret(n; \Theta_*)
  = \E{}{\sum_{t = 1}^n r(X_t, A_{t, *}; \Theta_*) - r(X_t, A_t; \Theta_*)}\,,
\end{align*}
where $A_{t, *} = \argmax_{a \in \cA} r(X_t, a; \Theta_*)$ is the optimal action in round $t$ given context $X_t$. In this work, we assume that the parameters $\Theta_*$ are also random. We define the $n$-round \emph{Bayes regret} as $\Bregret(n) = \E{}{\Regret(n; \Theta_*)}$, which takes an additional expectation over $\Theta_*$. While weaker than a traditional frequentist regret $\Regret(n; \Theta_*)$, the Bayes regret is a practical performance measure, when the average performance across multiple instances of model parameters is of interest \cite{russo14learning,hong20latent}.

\begin{figure}[t]
  \centering
  \includegraphics[width=.9\linewidth]{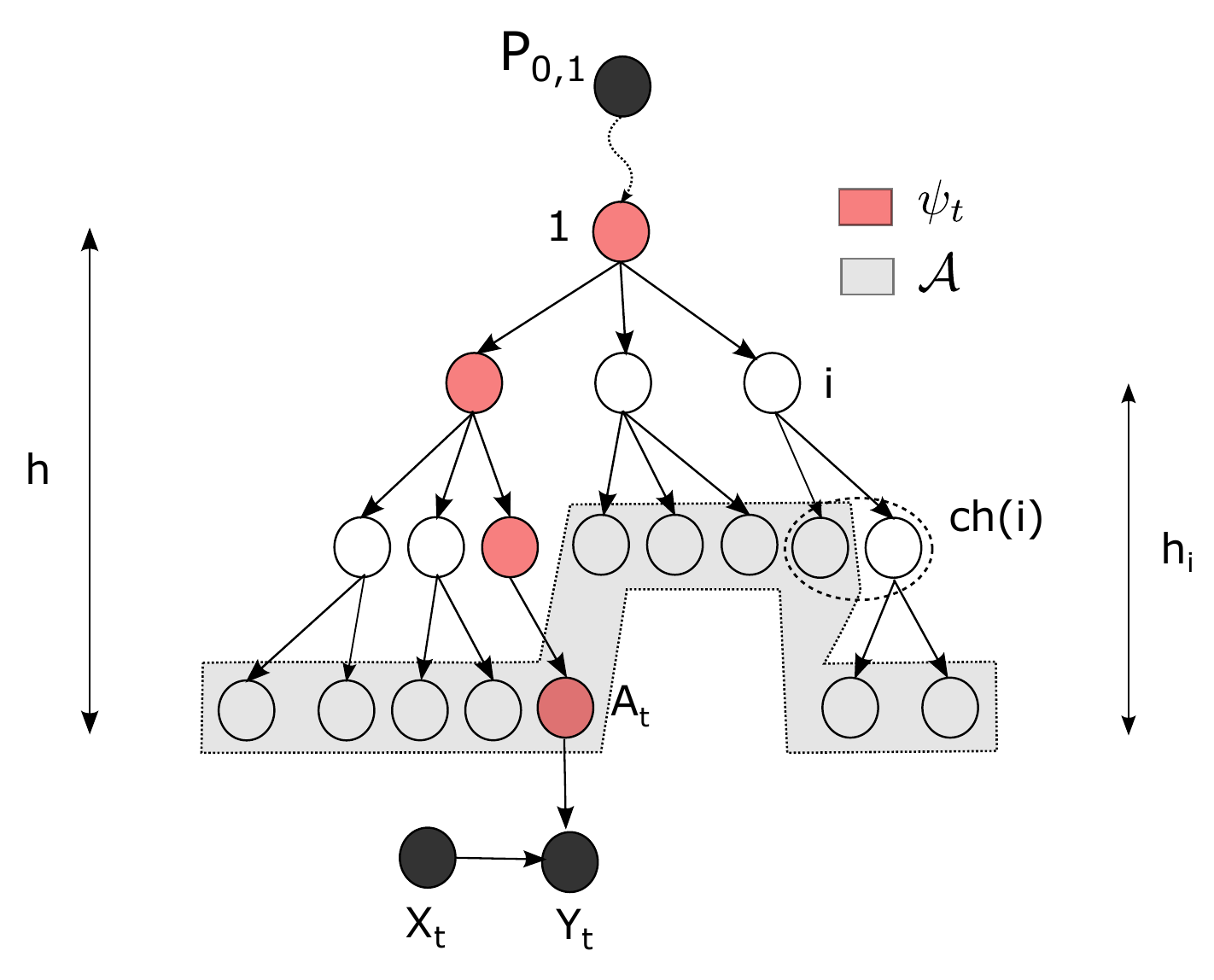}
  \vspace{-0.1in}
  \caption{Graphical model of our environment. The drawing depicts our notation: children $\children(i)$ of node $i$, action nodes $\cA$, and updated nodes $\psi_t$ after action $A_t$ is taken. The height of the tree is $h = 3$ and that of node $i$ is $h_i = 2$.}
  \label{fig:hierarchy}
\end{figure}

%% file: Algorithm.tex
\section{Algorithm}
\label{sec:algorithm}

Since our environment is a graphical model (\cref{fig:hierarchy}), we explore using \emph{Thompson sampling (TS)} \citep{thompson33likelihood,chapelle11empirical,agrawal12analysis,russo14learning}. The main challenge in our algorithm design are \emph{latent variables}. Specifically, the rewards of action nodes are observed and permit direct learning of $\theta_{*, a}$ for $a \in \cA$. However, parameters $\theta_{*, i}$ of the internal nodes $i \in \cV \setminus \cA$ are only indirectly observed through their descendant action nodes, and thus are latent.

It is unclear if modeling of latent variables is needed. To see this, let $H_t = (X_\ell, A_\ell, Y_\ell)_{\ell \in [t - 1]}$ be the \emph{history} of all interactions of the agent until round $t$ and $\Theta_{*, \cA} = (\theta_{*, a})_{a \in \cA}$ be the true model parameters corresponding to the action nodes. Since the mean rewards of actions depend only on $\Theta_{*, \cA}$, the most natural approach is to implement TS with posterior sampling of $\Theta_{*, \cA} \mid H_t$. This approach has two challenges. First, the exact posterior involves complex correlations, due to the dependencies in $\theta_{*, a}$ induced by the generative process in \eqref{eq:hierarchy}. These correlations remain when the latent variables are marginalized out. Thus exact sampling from $\Theta_{*, \cA} \mid H_t$ may be computationally inefficient. Second, the uncertainty of each $\theta_{*, a} \mid H_t$ could be modeled individually. While computationally efficient, this may not be sound and would not be statistically efficient. In contrast, we propose \emph{exact sampling} from $\Theta_{*, \cA} \mid H_t$ that is both \emph{computationally and statistically efficient}.

\subsection{Hierarchical Sampling}
\label{sec:hierarchical sampling}

Our approach is based on hierarchical sampling \citep{andrieu03introduction,doucet01montecarlo}, where the model parameters $\Theta$ are sampled similarly to the generative process in \eqref{eq:hierarchy}. To explain it, we introduce the following notation. For the root node, $P_{t, 1}(\theta) = \prob{\theta_{*, 1} = \theta \mid H_t}$ denotes the posterior distribution of its parameter in round $t$, which we also call a \emph{hyper-posterior}. For any other node $i$,
\begin{align}
  P_{t, i}(\theta \mid \theta_p)
  = \prob{\theta_{*, i} = \theta \mid \theta_{*, \parents(i)} = \theta_p, H_{t, i}}
  \label{eq:posterior}
\end{align}
is the posterior distribution of its parameter conditioned on $\theta_{*, \parents(i)} = \theta_p$ in round $t$, where $H_{t, i}$ is a subset of interactions in $H_t = (X_\ell, A_\ell, Y_\ell)_{\ell \in [t - 1]}$ where $A_\ell$ is a descendant of node $i$. In \eqref{eq:posterior}, $H_t$ could be replaced by $H_{t, i}$ because the posterior of $\theta_{*, i}$ is independent of the other observations given the value of the parent parameter $\theta_p$. This structure is critical to the computational efficiency of our approach and is also used in the regret analysis (\cref{sec:analysis}).

\begin{algorithm}[t]
  \caption{\hierts: Hierarchical Thompson sampling.}
  \label{alg:hierts}
  \begin{algorithmic}
    \STATE {\bfseries Input:} Tree $\cT$ with height $h$, all priors $P_{0, \cdot}$ in \eqref{eq:hierarchy}
    \STATE Initialize all posteriors $P_{1, \cdot} \gets P_{0, \cdot}$
    \FOR{$t = 1, \dots, n$}
      \STATE Sample $\theta_{t, 1} \sim P_{t, 1}$
      \FOR{$\ell = h - 2, \dots, 0$}
      \FOR{$i \in \cV_\ell$}
        \STATE Sample $\theta_{t, i} \sim P_{t, i}(\cdot \mid \theta_{t, \parents(i)})$
      \ENDFOR
      \ENDFOR
      \STATE $\Theta_t \gets (\theta_{t, i})_{i \in \cV}$
      \STATE Take action $A_t \gets \argmax_{a \in \cA} r(X_t, a; \Theta_t)$
      \STATE Observe $Y_t \sim P(\cdot \mid X_t; \theta_{*, A_t})$
      \STATE Compute new posteriors $P_{t + 1, \cdot}$
    \ENDFOR
  \end{algorithmic}
\end{algorithm}

Assuming that all posteriors $P_{t, i}$ can be computed efficiently, it is trivial to propose a hierarchical Thompson sampling algorithm for our problem. We call it \hierts and present its pseudo-code in \cref{alg:hierts}. In round $t$, \hierts works as follows. First, we sample the root parameter $\theta_{t, 1}$. After that, we iterate over all nodes and sample node parameters whose parents are already sampled. Specifically, we define $\cV_\ell = \{i \in \cV: h_i = \ell\}$ as the subset of nodes at height $\ell$ and then sample $\theta_{i, t}$ for $i \in \cV_\ell$, from the children of the root at height $\ell = h - 2$ to the leaves at $\ell = 0$. By definition, $\Theta_t = (\theta_{t, i})_{i \in \cV}$ is a valid posterior sample, generated hierarchically. Finally, \hierts takes an optimistic action with respect to $\Theta_t$, observes $Y_t$, and updates its posterior.

Note that \hierts samples parameters at all action nodes. It is possible to leverage the tree structure to prune sub-trees with actions that are unlikely to have high mean rewards. For example, \citet{sen2021topk} propose beam search over a tree to only evaluate a subset of actions. Such computational improvements can be easily incorporated into \hierts. We view them as orthogonal to our main contribution, which is statistically-efficient exploration using the tree structure.

\subsection{Efficient Posterior Computation}
\label{sec:efficient posterior computation}

The main technical novelty in \hierts is that the posteriors $P_{t, i}$ can be maintained efficiently. We show it as follows.

Fix any node $i$, its value $\theta$, and the value of its parent $\theta_p$. By Bayes rule, we have
\begin{align}
  P_{t, i}(\theta \mid \theta_p)
  \propto \cL_{t, i}(\theta) P_{0, i}(\theta \mid \theta_p)\,,
  \label{eq:efficient posterior}
\end{align}
where $\cL_{t, i}(\theta) = \prob{H_{t, i} \mid \theta_{*, i} = \theta}$ is the likelihood of observations $H_{t, i}$ whose ancestor is node $i$, given its value $\theta$. Note that $\cL_{t, i}(\theta)$ can be further decomposed as
\begin{align}
  \textstyle
  \cL_{t, i}(\theta)
  = \prod_{j \in \children(i)} \tilde{\cL}_{t, j}(\theta)\,,
  \label{eq:data likelihood}
\end{align}
where $\tilde{\cL}_{t, j}(\theta) = \prob{H_{t, j} \mid \theta_{*, \parents(j)} = \theta}$ is the likelihood of observations $H_{t, j}$ whose ancestor is child node $j$, given that the value of its parent is $\theta$. This identity follows from two facts. First, $\theta_{*, j}$ are conditionally independent of each other given $\theta_{*, i} = \theta$. Second, $H_{t, j}$ depends on $\theta_{*, i}$ only through $\theta_{*, j}$. Loosely speaking, each $\tilde{\cL}_{t, j}(\theta)$ in \eqref{eq:data likelihood} can be viewed as the likelihood of an aggregate observation at node $j$, from all leaves that descend from node $j$, under the hypothesis that $\theta_{*, i} = \theta$ (\cref{sec:mab posteriors}).

Finally, each $\tilde{\cL}_{t, j}(\theta)$ can be computed as
\begin{align}
  \tilde{\cL}_{t, j}(\theta)
  = \int_{\theta'}
  \cL_{t, j}(\theta') P_{0, j}(\theta' \mid \theta) \dif \theta'\,,
  \label{eq:child likelihood}
\end{align}
where $\cL_{t, j}(\theta')$ is the likelihood of observations $H_{t, j}$ whose ancestor is node $j$ given its value $\theta'$. Note that $\cL_{t, j}(\theta')$ can be further decomposed as in \eqref{eq:data likelihood}, which gives rise to our recursive computation of the posterior.

\begin{algorithm}[t]
  \caption{Statistics update after round $t$. The dot notation means that the likelihoods are updated for all parameter values, which is possible in Gaussian models (\cref{sec:posteriors}).}
  \label{alg:update}
  \begin{algorithmic}
    \STATE Initialize $\cL_{t + 1, \cdot} \gets \cL_{t, \cdot}$
    \STATE $i \gets A_t$
    \STATE $\cL_{t + 1, i}(\cdot)
    \gets P(Y_t \mid X_t; \cdot) \cL_{t, i}(\cdot)$
    \REPEAT 
      \STATE $i \gets \parents(i)$
      \STATE $\cL_{t + 1, i}(\cdot)
      \gets \prod_{j \in \children(i)} \tilde{\cL}_{t + 1, j}(\cdot)$
      \IF{$i > 1$}
        \STATE $\tilde{\cL}_{t + 1, i}(\cdot)
        \gets \int_\theta
        \cL_{t + 1, i}(\theta) P_{0, i}(\theta \mid \cdot) \dif \theta$
      \ENDIF
   \UNTIL{$i = 1$}
  \end{algorithmic}
\end{algorithm}

The pseudo-code for updating $\cL_{t, i}$ and $\tilde{\cL}_{t, i}$ after round $t$ is shown in \cref{alg:update}. After that, \eqref{eq:efficient posterior} has to be recomputed for all nodes $i$ on the path from $A_t$ to the root. In general, \eqref{eq:child likelihood} is hard to compute due to the integral over $\theta'$. However, in Gaussian graphical models (\cref{sec:posteriors}), this can be done in a closed form. In practice, \eqref{eq:child likelihood} can be approximated for arbitrary distributions using approximate inference, either variational or MCMC \citep{doucet01montecarlo}.

%% file: Posteriors.tex
\section{Gaussian Hierarchy}
\label{sec:posteriors}

In this section, we instantiate the environment in \eqref{eq:hierarchy} as a hierarchical Gaussian model \citep{koller09probabilistic} and derive its posterior. The model is defined as
\begin{align}
  \theta_{*, 1}
  & \sim \cN(\mu_1, \Sigma_{0, 1})\,,
  \label{eq:gaussian hierarchy} \\
  \theta_{*, i} \mid \theta_{*, \parents(i)}
  & \sim \cN(\theta_{*, \parents(i)}, \Sigma_{0, i})\,,
  & \forall i \in \cV \setminus \{1\}\,,
  \nonumber \\
  Y_t \mid X_t, \theta_{*, A_t} 
  & \sim \cN(X_t\T \theta_{*, A_t}, \sigma^2)\,,
  & \forall t \in [n]\,.
  \nonumber
\end{align}
Here $\theta_{*, i} \in \realset^d$ are node parameters and $\Sigma_{0, i}$ are covariance matrices that control the closeness of $\theta_{*, i}$ and $\theta_{*, \parents(i)}$. The mean reward of action $a$ in context $x$ is $r(x, a; \Theta) = x\T \theta_a$. The hierarchical structure is motivated by multi-label classification \citep{prabhu2018parabel,yu2020pecos}, where $X_t$ is a feature vector, $A_t$ is its predicted label, and $Y_t$ indicates if the label is correct. We return to this application in \cref{sec:image classification}. When $d = 1$ and $X_t = 1$, we recover a $K$-armed Gaussian bandit, where $\theta_{*, a}$ is the mean reward of action $a$. We assume that the agent knows the hyper-prior mean $\mu_1$, all covariances $\Sigma_{0, i}$, and noise $\sigma$. These assumptions are only used in the regret analysis, where we assume exact posterior sampling. In our experiments (\cref{sec:image classification}), we learn these quantities from data.

The special case of multi-armed bandits also shows computational gains over naive posterior sampling. Specifically, due to the dependencies in \eqref{eq:hierarchy}, $\Theta_{*, \cA} \mid H_t$ is a multivariate Gaussian with a $K \times K$ covariance matrix. Sampling from it requires $O(K^3)$ time, and can be done by computing the root of the covariance matrix and then multiplying it by a vector of $K$ standard normal variables. In contrast, sampling in \hierts requires only $O(|\cV|)$ time. When each internal node of $\cT$ has at least $2$ children, which is without loss of generality, $|\cV| \leq 2 K$ and our computational gain is $O(K^2)$. Now we present closed-form posteriors for hierarchies of Gaussian and contextual linear models.

\subsection{Multi-Armed Bandit}
\label{sec:mab posteriors}

We start with a $K$-armed Gaussian bandit. In this setting, each node $i \in \cV$ is associated with a single scalar parameter $\theta_{*, i} \in \realset$, and its initial uncertainty is described by conditional prior variance $\Sigma_{0, i} = \sigma_{0, i}^2 \in \realset$. The posteriors for this model are derived in \cref{sec:mab posterior derivation} and stated below.

The posterior of $\theta_{*, i}$ conditioned on $\theta_{*, \parents(i)} = \theta_p$, where $\theta_p$ is any scalar, is $P_{t, i}(\theta \mid \theta_p) = \cN(\theta; \hat{\theta}_{t, i}, \hat{\sigma}_{t, i}^2)$, where
\begin{align}
  \hat{\sigma}_{t, i}^{-2}
  & = \sigma_{0, i}^{-2} +
  \sum_{j \in \children(i)} \tilde{\sigma}_{t, j}^{-2}\,,
  \label{eq:mab_posterior_root} \\
  \hat{\theta}_{t, i}
  & = \hat{\sigma}_{t, i}^2 \bigg(\sigma_{0, i}^{-2} \theta_p +
  \sum_{j \in \children(i)} \tilde{\sigma}_{t, j}^{-2} \tilde{\theta}_{t, j}\bigg)\,.
  \nonumber
\end{align}
When $i = 1$ is the root, $\theta_p = \mu_1$. The child parameters $\tilde{\theta}_{t, j}$ and $\tilde{\sigma}_{t, j}$ are computed recursively as follows. If node $j$ is an action node, then
\begin{align}
  \tilde{\sigma}_{t, j}^2
  = \sigma_{0, j}^2 + \frac{\sigma^2}{\abs{\cS_{t, j}}}\,, \quad
  \tilde{\theta}_{t, j}
  = \frac{1}{\abs{\cS_{t, j}}} \sum_{\ell \in \cS_{t, j}} Y_\ell\,,
  \label{eq:mab_likelihood_leaf}
\end{align}
where $\cS_{t, j} = \{\ell < t: A_\ell = j\}$ are the rounds where action $j$ is taken before round $t$. If node $j$ is a non-action node,
\begin{align}
  \!\!\!\!\!\!
  \tilde{\sigma}_{t, j}^2
  = \sigma_{0, j}^2 + M^{-1}\,, \quad
  \tilde{\theta}_{t, j}
  = M^{-1} \sum_{k \in \children(j)} \tilde{\sigma}_{t, k}^{-2} \tilde{\theta}_{t, k}\,,
  \label{eq:mab_likelihood}
\end{align}
where $M = \sum_{k \in \children(j)} \tilde{\sigma}_{t, k}^{-2}$. The new child parameters $\tilde{\theta}_{t, k}$ and $\tilde{\sigma}_{t, k}$ are computed recursively, using either \eqref{eq:mab_likelihood_leaf} or \eqref{eq:mab_likelihood}.

At a high level, the recursive update follows from the observation that $\tilde{\cL}_{t, j}(\theta) \propto \exp\left[- \frac{1}{2} \tilde{\sigma}_{t, j}^{-2} (\theta - \tilde{\theta}_{t, j})^2\right]$ holds for any node $j$ and the value of its parent $\theta$. The closed-form of $P_{t, i}(\theta \mid \theta_p)$ is a direct combination of this result and the derivations in \cref{sec:efficient posterior computation}.

The recursive update in \eqref{eq:mab_likelihood} has an intuitive interpretation. Although we get Gaussian observations at action nodes, as in \eqref{eq:mab_likelihood_leaf}, they propagate to higher nodes in the tree through \eqref{eq:mab_likelihood}. In turn, these nodes act as noisy observations of their parents with mean $\tilde{\theta}_{t, j}$ and variance $\tilde{\sigma}_{t, j}^2$. This allows us to overcome latent variables in our model. The posterior in \eqref{eq:mab_posterior_root} is just a function of higher-level observations in all children of node $i$. To have closed forms of these quantities, we rely heavily on the properties of Gaussian random variables.

\subsection{Contextual Linear Bandit}
\label{sec:linear posteriors}

We now consider the general case in \eqref{eq:gaussian hierarchy}. This model can be viewed as a hierarchy of linear models \citep{yue11linear,abbasi-yadkori11improved} indexed by actions. We adopt the notation that $\Lambda = \Sigma^{-1}$, where $\Lambda$ is the precision for covariance matrix $\Sigma$. The posteriors for this model are derived in \cref{sec:linear posterior derivation} and stated below.

The posterior of $\theta_{*, i}$ conditioned on $\theta_{*, \parents(i)} = \theta_p$, where $\theta_p$ is any vector, is $P_{t, i}(\theta \mid \theta_p) = \cN(\theta; \hat{\theta}_{t, i}, \hat{\Sigma}_{t, i})$, where
\begin{align}
  \hat{\Lambda}_{t, i}
  & = \Lambda_{0, i} +
  \sum_{j \in \children(i)} \tilde{\Lambda}_{t, j}\,,
  \label{eq:linear_posterior_root} \\
  \hat{\theta}_{t, i}
  & = \hat{\Sigma}_{t, i} \bigg(\Lambda_{0, i} \theta_p +
  \sum_{j \in \children(i)} \tilde{\Lambda}_{t, j} \tilde{\theta}_{t, j}\bigg)\,.
  \nonumber
\end{align}
When $i = 1$ is the root, $\theta_p = \mu_1$. The child parameters $\tilde{\theta}_{t, j}$ and $\tilde{\Lambda}_{t, j}$ are computed recursively as follows. If node $j$ is an action node, then
\begin{align*}
  \tilde{\Sigma}_{t, j}
  = \Sigma_{0, j} + G_{t, j}^{-1}\,, \quad
  \tilde{\theta}_{t, j}
  = \sigma^{-2} G_{t, j}^{-1} \sum_{\ell \in \cS_{t, j}} X_\ell Y_\ell\,,
\end{align*}
where $\cS_{t, j} = \{\ell < t: A_\ell = j\}$ are the rounds where action $j$ is taken before round $t$ and $G_{t, j} = \sigma^{-2} \sum_{\ell \in \cS_{t, j}} X_\ell\T X_\ell$ is the outer product of the corresponding feature vectors. If node $j$ is a non-action node, then
\begin{align*}
  \tilde{\Sigma}_{t, j}
  = \Sigma_{0, j} + M^{-1}\,, \quad
  \tilde{\theta}_{t, j}
  = M^{-1} \sum_{k \in \children(j)} \tilde{\Lambda}_{t, j} \tilde{\theta}_{t, k}\,,
\end{align*}
where $M = \sum_{k \in \children(j)} \tilde{\Lambda}_{t, k}$. The new child parameters $\tilde{\theta}_{t, k}$ and $\tilde{\Lambda}_{t, k}$ are computed recursively, depending on whether $k$ is an action node or not.

At a high level, the recursive update follows from the fact that $\tilde{\cL}_{t, j}(\theta) \propto \exp\left[- \frac{1}{2} (\theta - \tilde{\theta}_{t, j})\T \tilde{\Lambda}_{t, j} (\theta - \tilde{\theta}_{t, j})\right]$ holds for any node $j$ and the value of its parent $\theta$. The closed-form of $P_{t, i}(\theta \mid \theta_p)$ is a direct combination of this result and the derivations in \cref{sec:efficient posterior computation}. As in \cref{sec:mab posteriors}, our recursive updates can be viewed as propagating observations from action nodes to higher nodes in the tree.

%% file: Analysis.tex
\section{Analysis}
\label{sec:analysis}

The analysis is for the Gaussian model in \cref{sec:mab posteriors}. We present the key lemmas, the main result, and discuss them. All proof are deferred to \cref{sec:regret bound proofs}.

\subsection{Key Steps in the Analysis}
\label{sec:key steps}

We start with the observation that the hierarchical posterior sampling in \cref{sec:hierarchical sampling} is just an efficient implementation of joint posterior sampling over the action node parameters $\Theta_\cA$. Since our model is a \emph{Gaussian graphical model}, this posterior is a multivariate Gaussian \citep{koller09probabilistic}. This is because any conditioning or marginalization does not change the model class. This observation allows us to prove the following lemma.

\begin{lemma}
\label{lem:bayes regret} For any $\delta > 0$, the Bayes regret of \hierts is bounded as 
\begin{align*}
  \Bregret(n)
  \leq \sqrt{2 n \cG(n) \log(1 / \delta)} +
  \sqrt{2 / \pi} \sigma_{\max} K n \delta\,,
\end{align*}
where $\cG(n) = \E{}{\sum_{t = 1}^n \sigma_{t, A_t}^2}$ denotes a \emph{complexity term}, $\sigma_{t, A_t}^2 = \condvar{\theta_{A_t}}{H_t}$ is the marginal posterior variance of the mean reward of action $A_t$ in round $t$, and $\sigma_{\max}$ is the maximum marginal prior width at an action node.
\end{lemma}

The second term in \cref{lem:bayes regret} is constant in $n$ for $\delta = 1 / n$. Therefore, we focus on the first $\tilde{O}(\sqrt{n})$ term. Also note that $\sigma_{t, A_t}^2$ in $\cG(n)$ is not the conditional posterior variance in \eqref{eq:mab_posterior_root}. We show how to decompose it into those variances at the updated nodes in round $t$ next.

To relate the marginal posterior variance, which is proportional to the instantaneous regret, to conditional posterior variances, which represent our model uncertainty, we adopt the following update-centric notation. We denote the list of nodes from the root to the action node $A_t$ in round $t$ by $\psi_t$. The length of $\psi_t$ is $L_t$. As an example of the notation, $\psi_t(1) = 1$ is the root, $\psi_t(L_t) = A_t$ is the action node, and $\psi_t(L_t - 1) = \parents(A_t)$ is its parent. \cref{fig:hierarchy} visualizes $\psi_t$. Now we are ready to relate the two quantities.

\begin{lemma}
\label{lem:regret decomposition} In any round $t$, the marginal posterior variance in action node $A_t$ decomposes as
\begin{align*}
  \sigma_{t, A_t}^2
  = \sum_{i = 1}^{L_t} \bigg(\prod_{j = i + 1}^{L_t}
  \frac{\hat{\sigma}_{t, \psi_t(j)}^4}{\sigma_{0, \psi_t(j)}^4}\bigg)
  \hat{\sigma}_{t, \psi_t(i)}^2\,.
\end{align*}
\end{lemma}

The last piece is a lower bound, which shows that each term in \cref{lem:regret decomposition} can be bounded by the posterior update of the corresponding node $i$, representing our information gain.

\begin{lemma}
\label{lem:posterior lower bound} Fix any round $t$ and $i \in [L_t]$. Then
\begin{align*}
  \hat{\sigma}_{t + 1, \psi_t(i)}^{-2} - \hat{\sigma}_{t, \psi_t(i)}^{-2}
  \geq c^{i - L_t} \bigg(\prod_{j = i + 1}^{L_t}
  \frac{\hat{\sigma}_{t, \psi_t(j)}^4}{\sigma_{0, \psi_t(j)}^4}\bigg)
  \sigma^{-2}\,,
\end{align*}
where $c > 0$ is a universal constant such that
\begin{align}
  c \hat{\sigma}_{t, i}^{-2}
  \geq \hat{\sigma}_{t + 1, i}^{-2}
  \label{eq:posterior scaling}
\end{align}
holds for any node $i$ and round $t$.
\end{lemma}

The condition in \eqref{eq:posterior scaling} means that the pseudo-counts before and after the posterior update do not change much, for any node $i$ and round $t$. The tightness of the bound is reflected by $c$ in \eqref{eq:posterior scaling}, which we bound next.

\begin{lemma}
\label{lem:posterior scaling} Let $\sigma_{0, \max} = \max_{i \in \cV} \sigma_{0, i}$. Then \eqref{eq:posterior scaling} holds for $c = 1 + \sigma_{0, \max}^2 / \sigma^2$. Moreover, $c = 2$ when $\sigma \geq \sigma_{0, \max}$.
\end{lemma}

\cref{lem:posterior scaling} shows that $c$ is controlled under reasonable assumptions, that the observation noise is higher than prior widths in $\cT$. If this is not the case, this property could be attained by initial forced exploration of all actions in $\cA$.

\subsection{Regret Bound}
\label{sec:regret bound}

Now we are ready to present our main result. Recall that $h$ is the height of $\cT$, $h_i$ is the height of node $i$, and that the action nodes have height $0$ (\cref{sec:setting}).

\begin{theorem}
\label{thm:regret bound} For any $\delta > 0$, the Bayes regret of \hierts is bounded as
\begin{align*}
  \Bregret(n)
  \leq \sqrt{2 n \cG(n) \log(1 / \delta)} +
  \sqrt{2 / \pi} \sigma_{\max} K n \delta\,,
\end{align*}
where $\cG(n) = \sum_{i \in \cV} c^{h_i} w_i$ and $c$ is a scalar defined in \eqref{eq:posterior scaling}. For an action node $i$, $h_i = 0$ and
\begin{align*}
  w_i
  = \frac{\sigma_{0, i}^2}{\log\left(1 + \frac{\sigma_{0, i}^2}{\sigma^2}\right)}
  \log\bigg(1 + \frac{\sigma_{0, i}^2 n}{\sigma^2}\bigg)\,.
\end{align*}
For a non-action node $i$, $h_i > 0$ and
\begin{align*}
  w_i
  = \frac{\sigma_{0, i}^2}{\log\left(1 + \frac{\sigma_{0, i}^2}{\sigma^2}\right)}
  \log\bigg(1 + \sigma_{0, i}^2 \sum_{j \in \children(i)} \sigma_{0, j}^{-2}\bigg)\,.
\end{align*}
\end{theorem}

For $\delta = 1 / n$, the above regret bound is $\tilde{O}(\sqrt{n \abs{\cV}})$, where $n$ is the horizon and $\abs{\cV}$ is the number of nodes, and thus of learned parameters. The dependence on the horizon $n$ is standard. As $w_i = \tilde{O}(\sigma_{0, i}^2)$, the contribution of each node $i$ to the regret is proportional to its prior variance. Thus the regret decreases when the initial uncertainty is lower. One notable term in $\cG(n)$ is exponential scaling with height $c^{h_i}$. This is not problematic, as the number of higher nodes is exponentially smaller than the lower nodes (\cref{sec:lower regret}).

\cref{thm:regret bound} also recovers a well-known Bayes regret bound for $K$-armed bandits \citep{russo14learning}. The reason is that a $K$-armed bandit can be viewed as a tree with height $h = 1$, where the root parameter $\theta_1$ is the prior mean of the actions. Because $\theta_1$ is certain, $w_1 \to 0$ and $\sum_{i \in \cV} c^{h_i} w_i \approx \sum_{i = 2}^{K + 1} w_i = O(K)$.

\subsection{Lower Regret due to Hierarchy}
\label{sec:lower regret}

Now we give examples of how the hierarchy can help with reducing regret. To simplify the discussion, we ignore logarithmic factors in the definitions of $w_i$ in \cref{thm:regret bound}. We assume that $\cT$ is a balanced $b$-ary tree with height $h$; with $K = b^h$ action nodes and $b^{h - \ell}$ nodes at height $\ell$. Our discussion is under the assumption that $c = 2$, as derived in \cref{lem:posterior scaling}. More gains are possible when $c < 2$.

We compare the regret of \hierts to classical Thompson sampling (\ts), which ignores the hierarchy $\cT$, and maintains independent posteriors of $\theta_{*, a}$ for all actions $a \in \cA$. To have a fair comparison, we set the marginal prior variances of all actions in \ts as in \hierts. Specifically, let $\psi_a$ be the path in $\cT$ from action node $a$ to the root. Then the marginal prior of action $a$ is $\cN(\mu_1, \bar{\sigma}_{0, a}^2)$, where $\bar{\sigma}_{0, a}^2 = \sum_{i \in \psi_a} \sigma_{0, i}^2$ and $\mu_1$ is the hyper-prior mean. Note that \ts can be analyzed exactly as \hierts in \cref{thm:regret bound}. The only difference is in the complexity term $\cG_{\ts}(n) \approx \sum_{a \in \cA} \bar{\sigma}_{0, a}^2$.

\textbf{Problem 1.} We start with a problem where all prior variances are identical, $\sigma_{0, i}^2 = 1$ for any $i \in \cV$. In this case, all prior variances in \ts are $\bar{\sigma}_{0, a}^2 = h + 1$ and its complexity term is $\cG_{\ts}(n) = (h + 1) b^h$. In \hierts, we aggregate the nodes by height and get
\begin{align*}
  \cG(n)
  = \sum_{\ell = 0}^h b^{h - \ell} c^\ell
  = b^h \sum_{\ell = 0}^h (2 / b)^\ell
  \leq \frac{1}{1 - 2 / b} b^h\,.
\end{align*}
Thus \hierts improves $\cG(n)$ by $\Omega(h)$ when $b > 2$. Since $h = \log_b b^h = \log_b K$, we get $\cG_{\ts}(n) / \cG(n) \approx \log_b K$, and \hierts reduces the Bayes regret by a multiplicative factor $\sqrt{\log_b K}$. This argument can be adjusted for $b = 2$ to get a comparable regret to \ts.

\textbf{Problem 2.} Now we consider a problem where the conditional prior variances in $\cT$ double with height $\sigma_{0, i}^2 = 2^{h_i}$, where $h_i$ is the height of node $i$. This setting is motivated in \cref{sec:introduction}. We expect higher statistical gains because the uncertainty of highly-uncertain nodes at higher levels of $\cT$ is reduced jointly using all actions. In this case, all prior variances in \ts are $\bar{\sigma}_{0, a}^2 = 2^{h + 1}$ and its complexity term is $\cG_{\ts}(n) = 2^{h + 1} b^h$. In comparison, \hierts yields
\begin{align*}
  \cG(n)
  = \sum_{\ell = 0}^h b^{h - \ell} c^\ell 2^\ell
  \leq \sum_{\ell = 0}^h b^{h - \ell} 4^\ell
  \leq \frac{1}{1 - 4 / b} b^h\,,
\end{align*}
where the last step is by the same argument in Problem 1. Thus \hierts improves $\cG(n)$ by $\Omega(2^{h + 1})$ if $b > 4$. Since
\begin{align*}
  2^{h + 1}
  = 2 \cdot 2^{\log_b b^h}
  = 2 \cdot 2^{\log_2 b^h / \log_2 b}
  = 2 K^\frac{1}{\log_2 b}\,,
\end{align*}
we get $\cG_{\ts}(n) / \cG(n) \approx K^\frac{1}{\log_2 b}$, and \hierts reduces the Bayes regret by a multiplicative factor $\sqrt{K^\frac{1}{\log_2 b}}$. For $b = 5$, this factor would be close to $K^\frac{1}{4}$. Therefore, the regret is reduced by a polynomial factor in $K$.

%% file: Experiments.tex
\section{Experiments}
\label{sec:experiments}

\begin{figure*}[t!]
  \centering
  \begin{minipage}{0.31\textwidth}
    \includegraphics[width=\linewidth]{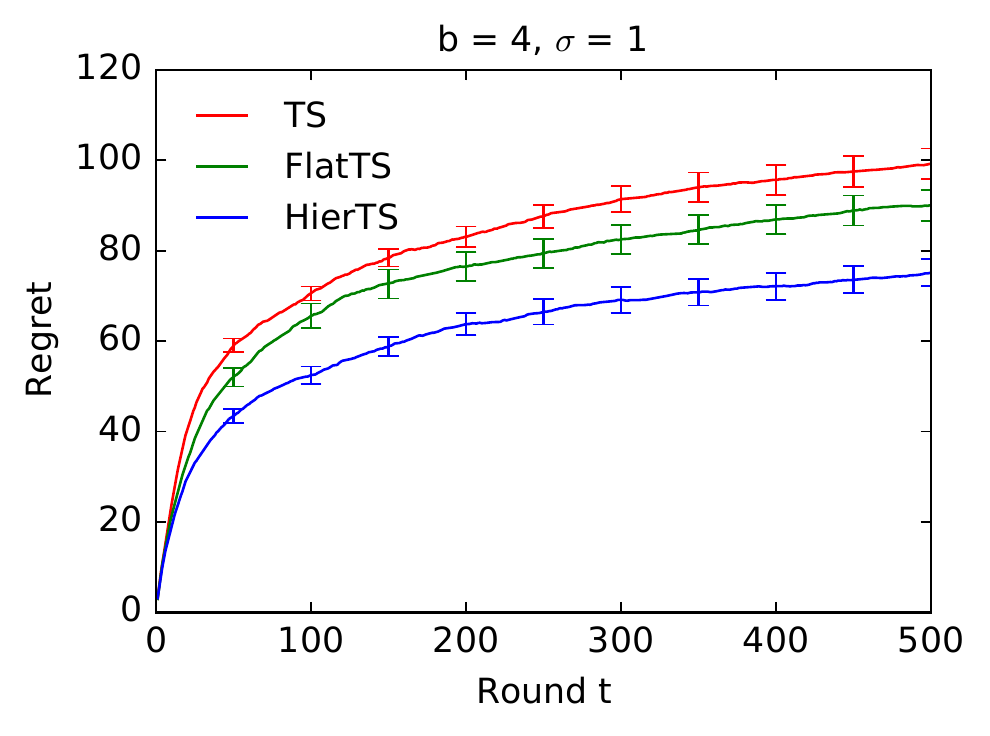}
  \end{minipage}
  \begin{minipage}{0.31\textwidth}
    \includegraphics[width=\linewidth]{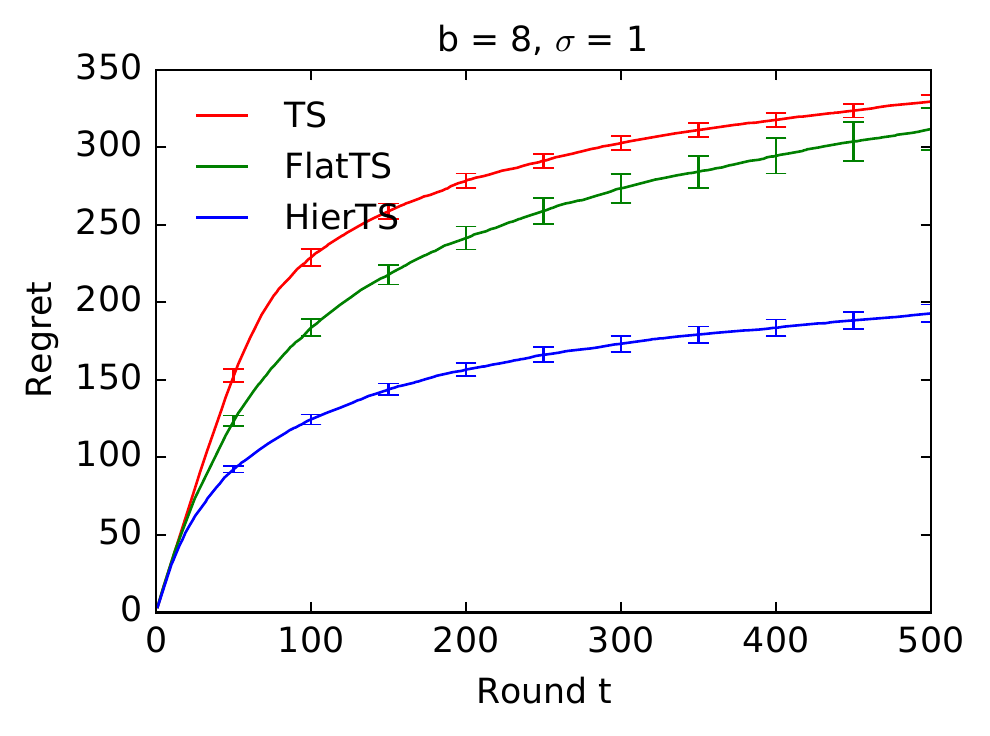}
  \end{minipage}
  \begin{minipage}{0.31\textwidth}
    \includegraphics[width=\linewidth]{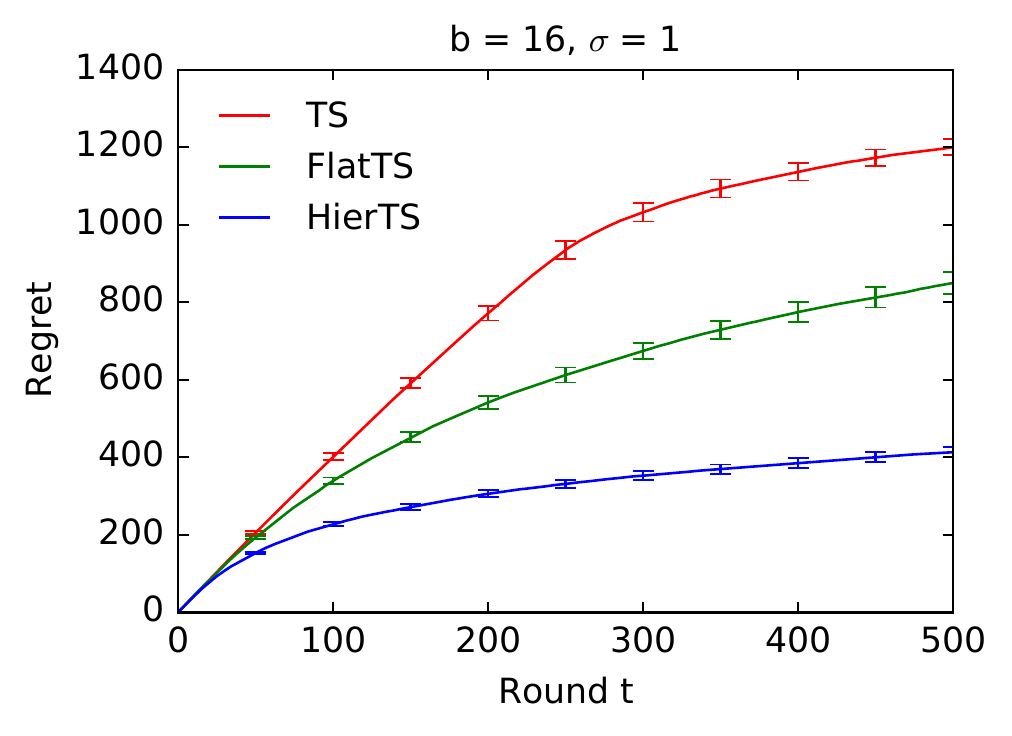}
  \end{minipage}
  \vspace{-0.1in}
  \caption{Regret of \hierts on synthetic bandit problems with varying branching factor $b$.}
  \label{fig:increasing_vary_b}
\end{figure*}

We compare \hierts to two Thompson sampling baselines that either ignore or only partially leverage the hierarchy $\cT$. The first baseline is classical TS (\ts), which treats each action independently and is introduced in \cref{sec:lower regret}. The second baseline, which we call \hiertstwo, only uses a $2$-level hierarchy $\cT$, namely the root and leaves. The hyper-prior $P_{0, 1}$ for the root is unchanged. However, for any action $a \in \cA$, the conditional prior is $P_{0, a}(\cdot \mid \theta_1) = \cN(\cdot; \theta_1, \bar{\sigma}_{0, a}^2 - \sigma_{0, 1}^2)$, where $\bar{\sigma}_{0, a}^2$ is the marginal prior variance used by \ts. This baseline mimics existing algorithms for $2$-level Gaussian hierarchies with a common root \citep{kveton21metats,basu21no,hong2022hierts}, and is similar to structured bandits where the actions share a latent parameter \citep{gupta2021unified}. In a contextual linear bandit, we implemented the baselines analogously, by replacing variances with covariances.

\subsection{Synthetic Experiments}
\label{sec:synthetic experiments}

Our first experiments are on a synthetic Gaussian bandit, where we validate theoretical findings from \cref{sec:lower regret}. We experiment with both problems in \cref{sec:lower regret}, which are $b$-ary trees with height $h$ and $K = b^h$ actions. In Problem 1, the prior variances are constant. In Problem 2, the prior variances double with height. In both problems, the mean of the hyper-prior is $\mu_1 = 0$, and the reward of action $a$ is $\theta_{*, a}$ with variance $\sigma^2 = 1$.

We start with Problem 2, where we fix the height at $h = 2$ and vary the branching factor $b$. All algorithms are run for $n = 500$ rounds and evaluated by the Bayes regret on $100$ independent samples of $\Theta_*$. We plot its mean and standard error in \cref{fig:increasing_vary_b}. For all $b$, we observe that \hierts significantly outperforms both baselines. In the next experiment, we consider both Problems 1 and 2. We fix the branching factor at $b = 2$ and vary the height $h$. All algorithms are run for $n = 500$ rounds on $100$ independent samples of $\Theta_*$. We measure the reduction in the Bayes regret of \hierts and \hiertstwo, as a ratio of the \ts regret over the regret in question. In \cref{fig:vary_h}a, we plot the ratios for Problem 1. \cref{sec:lower regret} suggests a $O(h)$ reduction in the \hierts regret, which the plot confirms. In \cref{fig:vary_h}b, we plot the ratios for Problem 2. \cref{sec:lower regret} suggests a $O(2^{h / 2})$ reduction in the \hierts regret, which is exponential in height $h$. The plot confirms this.

\subsection{Multi-Label Image Classification}
\label{sec:image classification}

\begin{figure*}[t]
  \centering
  \begin{minipage}{0.6\textwidth}
    \begin{minipage}{0.5\linewidth}
      \includegraphics[width=\linewidth]{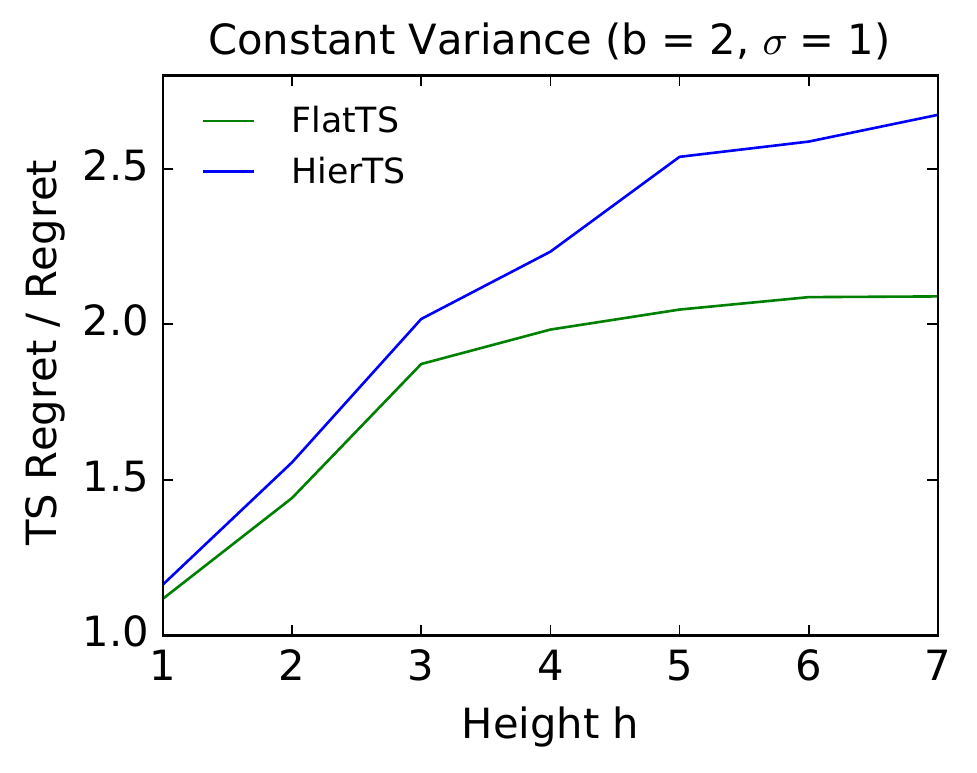}
    \end{minipage}
    \begin{minipage}{0.5\linewidth}
      \includegraphics[width=\linewidth]{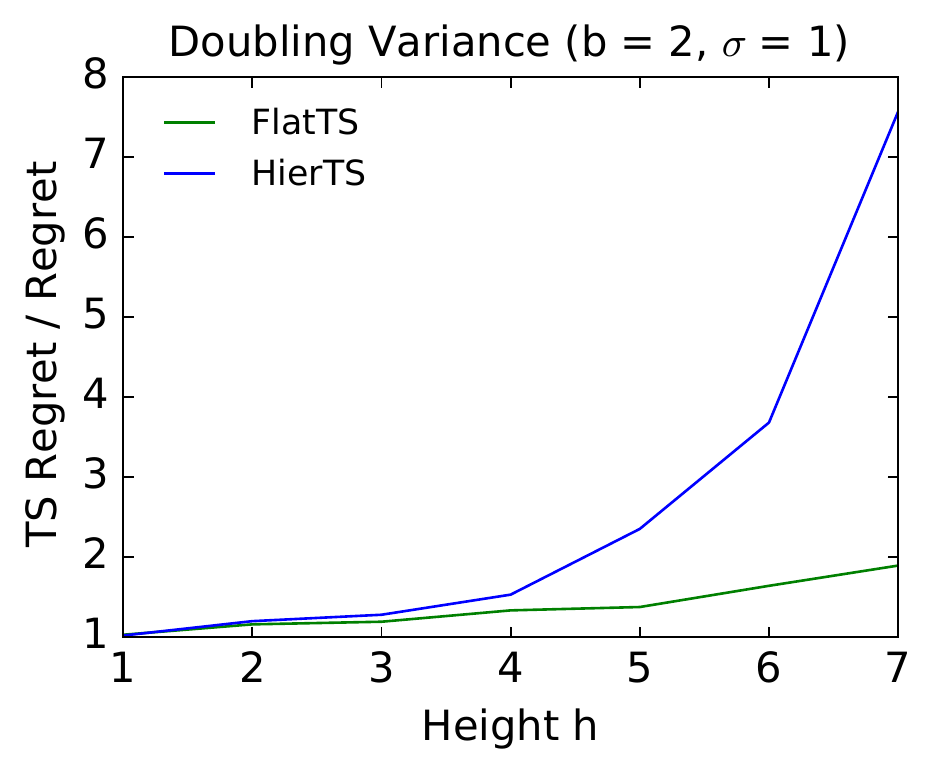}
    \end{minipage}
    \vspace{-0.1in}
    \caption{Improvement in the \hierts regret when the prior variance (a) is constant or (b) increases with height, as a function of the tree height $h$.}
    \label{fig:vary_h}
  \end{minipage}
  \hspace{0.1in}
  \begin{minipage}{0.33\textwidth}
    \includegraphics[width=\linewidth]{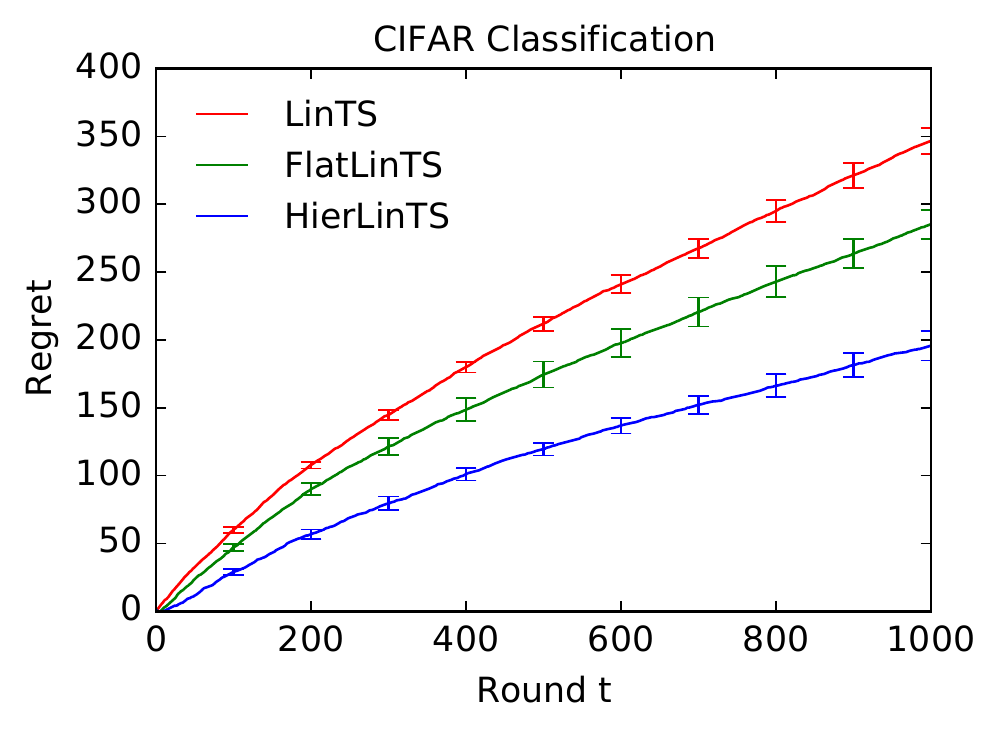}
    \vspace{-0.25in}
    \caption{Regret of \hierts on CIFAR-100 image classification.}
    \label{fig:cifar100}
  \end{minipage}
  \vspace{-0.1in}
\end{figure*}

The last experiment is on a multi-label image classification problem with linear rewards. We use the CIFAR-100 dataset \citep{cifar}, which comprises $60,000$ images of size $32 \times 32$. There are $50,000$ training and $10,000$ test images. Each image belongs to one of $100$ classes (labels) and $20$ super-classes, each consisting of $5$ classes. Each image is represented by a $d = 10$ dimensional feature vector, which we obtain by downsampling a $100$-dimensional feature vector. That one is an embedding computed by an EfficientNet-L2 network applied to the image \cite{Xie_2020_CVPR, efficientnet,foret2021sharpnessaware}. The network is a convolutional neural network pretrained on both ImageNet~\citep{ILSVRC15} and unlabeled JFT-300M~\citep{jft}, and fine-tuned on the CIFAR-100 training set.

We randomly select $5$ super-classes and their corresponding $K = 25$ classes become actions. The test and training sets are filtered to these classes. Our bandit problem is set up as follows. For each action $a$, $\theta_{*, a}$ is the mean feature vector of test images in class $a$. In round $t$, context $X_t$ is the feature vector of a random image from the test set and the reward for taking action $A_t$ is $Y_t \sim \cN(X_t\T \theta_{*, A_t}, 0.5^2)$. Therefore, on average, the reward is maximized when the true class is chosen. Finally, we construct a $3$-level hierarchy $\cT$ as follows. The hyper-prior for the root $P_{0, 1} = \cN(\mu_1, \Sigma_{0, 1})$ is a Gaussian fitted to all \emph{training images}. The nodes at height $1$ correspond to the $5$ super-classes and have conditional priors $P_{0, i}(\cdot \mid \theta_1) = \cN(\cdot; \theta_1, \Sigma_{0, i})$, where $\Sigma_{0, i}$ is fitted to the training images of super-class $i$. Finally, at height $2$, the nodes correspond to actions and have conditional priors $P_{0, a}(\cdot \mid \theta_{\parents(a)}) = \cN(\cdot; \theta_{\parents(a)}, \Sigma_{0, a})$, where $\Sigma_{0, a}$ is fitted to the training images of class $a$.

We report the mean and standard error of the regret over $10$ runs of linear variants of all algorithms in \cref{fig:cifar100}. We observe again that \hierts outperforms both baselines. Note that the true model parameters of $\cT$, namely $\mu_1$ and $\Sigma_{0, i}$, are unknown in this problem; and we estimate them from training images. Therefore, even when we relax the assumption that they are known, it is beneficial to estimate them, and use the structure of $\cT$.

%% file: RelatedWork.tex
\section{Related Work}
\label{sec:related work}

Thompson sampling algorithms have been widely applied to contextual bandits because of their computational efficiency and strong empirical performance \citep{chu11contextual,chapelle11empirical,abbasi-yadkori11improved}. \citet{russo14learning} derived first Bayes regret bounds for TS. Our proposed algorithm \hierts extends TS to tree hierarchies. TS with a $2$-level hierarchy over tasks was applied and analyzed in both meta-learning and multi-task learning \citep{kveton21metats,basu21no,wan21metadata,hong2022hierts}. The main difference in our work is that we move from $2$-level hierarchies to an arbitrary depth, and develop both algorithmic and theory foundations for this setting. Our analysis extends the variance decompositions proposed in \citet{hong2022hierts} to trees. Alternatively, information theory could be used to derive Bayes regret bounds \citep{russo16information,lu2019information,basu21no}, but we are unaware of any for trees.

We consider a setting where an underlying structure exists among the actions. There are prior works in structured bandits \citep{tirinzoni20novel,lattimore2014bounded,gupta2021unified} that assume a shared latent parameter among all actions. This can be viewed as a special case of our setting with a $2$-level hierarchy. In latent bandits, the parameter is a discrete variable \citep{maillard14latent,hong20latent}. Recent works also applied approximate TS to more complex structures \citep{gopalan14thompson,yu20graphical}. Such algorithms are general, but can only be analyzed in limited settings with strong assumptions. We consider a special tree structure, where we can derive and analyze an exact algorithm. \citet{majzoubi2020efficient} study a contextual bandit problem in continuous action spaces, where they utilize a hierarchical tree structure and propose a UCB algorithm. While UCB algorithms tend to be conservative, \hierts leverages the hierarchy and can be run as analyzed without tuning.

The closest related work is \citet{sen2021topk}, who also study contextual bandits with a tree hierarchy over actions. Both of our works address the structured action space using a hierarchy of regressors, motivated by multi-label classification \citep{prabhu2018parabel,yu2020pecos}. However, our work differs in several key aspects. First, we consider a distributional perspective over the action hierarchy, where the internal nodes are associated with a prior distribution rather than a fixed center and radius as in \citet{sen2021topk}. In fact, \citet{sen2021topk} do not model the statistical uncertainty at all. Second, we propose a TS algorithm using novel recursive derivations of the posterior. \citet{sen2021topk} consider a greedy strategy and beam search to avoid evaluation of all actions. In \cref{sec:algorithm}, we discuss how similar improvements can be incorporated in \hierts. Finally, our Bayesian analysis reveals structural properties that imply low regret. The low regret in \citet{sen2021topk} is attained by making an assumption on the regression oracle, and can grow linearly when the oracle is imperfect.

%% file: Conclusions.tex
\section{Conclusions}
\label{sec:conclusions}

In many practical problems, the action space is large and a good generalization over the actions is not obvious. Motivated by this, we study a contextual bandit problem with a \emph{deep hierarchy} over the actions. We propose hierarchical Thompson sampling (\hierts) for regret minimization in this model, which can be implemented exactly and efficiently with Gaussian observations. We prove a Bayes regret bound for \hierts that quantifies its increased statistical efficiency over vanilla TS. We validate this finding empirically, and also apply \hierts to a challenging classification problem with label hierarchy.

Our work is a major step towards studying bandit problems with rich graphical model structures. Many of its limitations can be addressed by future work. For instance, a frequentist analysis is possible and would only differ in \cref{lem:bayes regret}. The rest of the proof, which captures the structure of our problem, is a worst-case argument. In addition, it is easy to extend our analysis to contextual bandits, based on the similarities of the multi-armed (\cref{sec:mab posteriors}) and contextual (\cref{sec:linear posteriors}) bandit posteriors; and that the proof of \cref{thm:regret bound} relies on an elliptical-like lemma. Finally, we believe that our method can be extended beyond Gaussian trees. As discussed in \cref{sec:efficient posterior computation}, exact posterior sampling is challenging under the constraint of computational efficiency; but many tractable approximations exist. For exact posterior sampling, we believe that our proofs can be extended to general exponential-family distributions. Another direction of future work is an extension to directed acyclic graphs (DAGs). The nodes in DAGs can be ordered, and therefore similar recursions to \cref{sec:posteriors,sec:analysis} can be established.

%% file: AppendixA.tex
\section{Posterior Derivations}
\label{sec:posterior derivations}

This section contains our posterior derivations.

\subsection{Multi-Armed Bandit Posterior}
\label{sec:mab posterior derivation}

The proof is by induction. We start with the inductive step.

\begin{lemma}
\label{lem:mab recursion} Fix a non-action node $j$. Let $\parents(j) = i$ and $C = \children(j)$. Let
\begin{align}
  \condprob{H_{t, j}, \theta_{*, j} = \theta}{\theta_{*, i} = \theta_i}
  \propto \exp\left[- \frac{1}{2}
  \left(\sigma_0^{-2} (\theta - \theta_i)^2 +
  \sum_{k \in C} \sigma_k^{-2} (\theta - \theta_k)^2\right)\right]\,,
  \label{eq:mab induction}
\end{align}
where $\theta_k, \sigma_k$ are the parameters of $k \in C$. Then
\begin{align*}
  \condprob{H_{t, j}}{\theta_{*, i} = \theta_i}
  \propto \exp\left[- \frac{1}{2}
  \tilde{\sigma}_j^{-2} (\theta_i - \tilde{\theta}_j)^2\right]
\end{align*}
for some $\tilde{\theta}_j$ and $\tilde{\sigma}_j$.
\end{lemma}
\begin{proof}
Let $s = \sigma_0^{-2} + \sum_{k \in C} \sigma_k^{-2}$ and $v = \sigma_0^{-2} \theta_i + \sum_{k \in C} \sigma_k^{-2} \theta_k$. We start with completing the square of $\theta$,
\begin{align*}
  \log \condprob{H_{t, j}, \theta_{*, j} = \theta}{\theta_{*, i} = \theta_i}
  & \propto \sigma_0^{-2} (\theta - \theta_i)^2 +
  \sum_{k \in C} \sigma_k^{-2} (\theta - \theta_k)^2 \\
  & \propto s \theta^2 -
  2 \theta \left(\sigma_0^{-2} \theta_i + \sum_{k \in C} \sigma_k^{-2} \theta_k\right) +
  \sigma_0^{-2} \theta_i^2 \\
  & = s (\theta^2 - 2 \theta s^{-1} v + s^{-2} v^2) +
  \sigma_0^{-2} \theta_i^2 - s^{-1} v^2 \\
  & = s (\theta - s^{-1} v)^2 + \sigma_0^{-2} \theta_i^2 - s^{-1} v^2\,.
\end{align*}
In the second step, we omit constants in $\theta$ and $\theta_i$. Since we got a quadratic form in $\theta$, we know that
\begin{align*}
  \int_\theta \condprob{H_{t, j}, \theta_{*, j} = \theta}
  {\theta_{*, i} = \theta_i} \dif \theta
  \propto \exp\left[- \frac{1}{2}
  (\sigma_0^{-2} \theta_i^2 - s^{-1} v^2)\right]\,.
\end{align*}
Let $\hat{s} = \sigma_0^{-2} - \sigma_0^{-4} s^{-1}$. Now we complete the square of $\theta_i$,
\begin{align*}
  \sigma_0^{-2} \theta_i^2 - s^{-1} v^2
  & = \sigma_0^{-2} \theta_i^2 - s^{-1}
  \left(\sigma_0^{-2} \theta_i + \sum_{k \in C} \sigma_k^{-2} \theta_k\right)^2
  = \sigma_0^{-2} \theta_i^2 - \sigma_0^{-4} s^{-1}
  \left(\theta_i + \sigma_0^2 \sum_{k \in C} \sigma_k^{-2} \theta_k\right)^2 \\
  & \propto \hat{s} \left(\theta_i^2 -
  2 \theta_i \hat{s}^{-1} \sigma_0^{-2} s^{-1} \sum_{k \in C} \sigma_k^{-2} \theta_k\right)
  \propto \hat{s} \left(\theta_i -
  \hat{s}^{-1} \sigma_0^{-2} s^{-1} \sum_{k \in C} \sigma_k^{-2} \theta_k\right)^2\,.
\end{align*}
In the last two steps, we omit constants in $\theta_i$. Finally, note that
\begin{align*}
  \hat{s}
  & = \frac{\sigma_0^{-2} (s - \sigma_0^{-2})}{s}
  = (\sigma_0^2 + (s - \sigma_0^{-2})^{-1})^{-1}\,, \\
  \hat{s}^{-1} \sigma_0^{-2} s^{-1}
  & = \frac{s}{\sigma_0^{-2} (s - \sigma_0^{-2})} \sigma_0^{-2} s^{-1}
  = (s - \sigma_0^{-2})^{-1}\,.
\end{align*}
This completes the proof, for $\tilde{\theta}_j = (s - \sigma_0^{-2})^{-1} \sum_{k \in C} \sigma_k^{-2} \theta_k$ and $\tilde{\sigma}_j^2 = \sigma_0^2 + (s - \sigma_0^{-2})^{-1}$.
\end{proof}

For an action node $j$, \eqref{eq:mab induction} holds for $\sigma_k = \sigma$ and $\theta_k = Y_k$, where $Y_k$ is an observation $k$ of node $j$ and $\sigma$ is observation noise. This is the basis of the induction.

\subsection{Linear Bandit Posterior}
\label{sec:linear posterior derivation}

The proof is by induction. We start with the inductive step.

\begin{lemma}
\label{lem:linear recursion} Fix a non-action node $j$. Let $\parents(j) = i$ and $C = \children(j)$. Let
\begin{align}
  \condprob{H_{t, j}, \theta_{*, j} = \theta}{\theta_{*, i} = \theta_i}
  \propto \exp\left[- \frac{1}{2}
  \left((\theta - \theta_i)\T \Lambda_0 (\theta - \theta_i) +
  \sum_{k \in C} (\theta - \theta_k)\T \Lambda_k (\theta - \theta_k)\right)\right]\,,
  \label{eq:linear induction}
\end{align}
where $\theta_k, \Lambda_k$ are the parameters of $k \in C$. Then
\begin{align*}
  \condprob{H_{t, j}}{\theta_{*, i} = \theta_i}
  \propto \exp\left[- \frac{1}{2}
  (\theta_i - \tilde{\theta}_j)\T \tilde{\Lambda}_j (\theta_i - \tilde{\theta}_j)\right]
\end{align*}
for some $\tilde{\theta}_j$ and $\tilde{\Lambda}_j$.
\end{lemma}
\begin{proof}
Let $S = \Lambda_0 + \sum_{k \in C} \Lambda_k$ and $V = \Lambda_0 \theta_i + \sum_{k \in C} \Lambda_k \theta_k$. We start with completing the square of $\theta$,
\begin{align*}
  \log \condprob{H_{t, j}, \theta_{*, j} = \theta}{\theta_{*, i} = \theta_i}
  & \propto (\theta - \theta_i)\T \Lambda_0 (\theta - \theta_i) +
  \sum_{k \in C} (\theta - \theta_k)\T \Lambda_k (\theta - \theta_k) \\
  & \propto \theta\T S \theta -
  2 \theta\T \left(\Lambda_0 \theta_i + \sum_{k \in C} \Lambda_k \theta_k\right) +
  \theta_i\T \Lambda_0 \theta_i \\
  & = \theta\T S (\theta - 2 S^{-1} V) + \theta_i\T \Lambda_0 \theta_i \\
  & = (\theta - S^{-1} V)\T S (\theta - S^{-1} V) +
  \theta_i\T \Lambda_0 \theta_i - V\T S^{-1} V\,.
\end{align*}
In the second step, we omit constants in $\theta$ and $\theta_i$. Since we got a quadratic form in $\theta$, we know that
\begin{align*}
  \int_\theta \condprob{H_{t, j}, \theta_{*, j} = \theta}
  {\theta_{*, i} = \theta_i} \dif \theta
  \propto \exp\left[- \frac{1}{2}
  (\theta_i\T \Lambda_0 \theta_i - V\T S^{-1} V)\right]\,.
\end{align*}
Let $\hat{S} = \Lambda_0 - \Lambda_0 S^{-1} \Lambda_0$. Now we complete the square of $\theta_i$,
\begin{align*}
  \theta_i\T \Lambda_0 \theta_i - V\T S^{-1} V
  & = \theta_i\T \Lambda_0 \theta_i -
  \left(\Lambda_0 \theta_i + \sum_{k \in C} \Lambda_k \theta_k\right)\T S^{-1}
  \left(\Lambda_0 \theta_i + \sum_{k \in C} \Lambda_k \theta_k\right) \\
  & \propto \theta_i\T \hat{S} \left(\theta_i -
  2 \hat{S}^{-1} \Lambda_0 S^{-1} \sum_{k \in C} \Lambda_k \theta_k\right) \\
  & \propto
  \left(\theta_i - \hat{S}^{-1} \Lambda_0 S^{-1} \sum_{k \in C} \Lambda_k \theta_k\right)\T
  \hat{S}
  \left(\theta_i - \hat{S}^{-1} \Lambda_0 S^{-1} \sum_{k \in C} \Lambda_k \theta_k\right)\,.
\end{align*}
In the last two steps, we omit constants in $\theta_i$. Finally, by the Woodbury matrix identity, we have
\begin{align*}
  \hat{S}
  & = \Lambda_0 - \Lambda_0 S^{-1} \Lambda_0
  = (\Lambda_0^{-1} + (S - \Lambda_0)^{-1})^{-1}\,, \\
  \hat{S}^{-1} \Lambda_0 S^{-1}
  & = (\Lambda_0 - \Lambda_0 S^{-1} \Lambda_0)^{-1} \Lambda_0 S^{-1}
  = (S - \Lambda_0)^{-1}\,.
\end{align*}
This completes the proof, for $\tilde{\theta}_j = (S - \Lambda_0)^{-1} \sum_{k \in C} \Lambda_k \mu_k$ and $\tilde{\Lambda}_j = (\Lambda_0^{-1} + (S - \Lambda_0)^{-1})^{-1}$.
\end{proof}

For an action node $j$, note that \eqref{eq:linear induction} can be written as
\begin{align*}
  \condprob{H_{t, j}, \theta_{*, j} = \theta}{\theta_{*, i} = \theta_i}
  \propto \exp\left[- \frac{1}{2}
  \left((\theta - \theta_i)\T \Lambda_0 (\theta - \theta_i) +
  \sum_{k \in C} \theta\T \Lambda_k \theta -
  2 \sum_{k \in C} \theta\T \Lambda_k \theta_k)\right)\right]\,,
\end{align*}
when constants in $\theta$ and $\theta_i$ are omitted. Then $\Lambda_k = \sigma^{-2} X_k\T X_k$ and $\Lambda_k \theta_k = \sigma^{-2} X_k Y_k$, where $Y_k$ is an observation $k$ of node $j$ at feature vector $X_k$ and $\sigma$ is observation noise. This is the basis of the induction.

%% file: AppendixB.tex
\section{Regret Bound Proofs}
\label{sec:regret bound proofs}

This section contains proofs of our regret bound and supporting lemmas.

\subsection{Proof of \cref{lem:bayes regret}}

Fix round $t$. Let $\condprob{\Theta}{H_t} = \cN(\Theta; \hat{\Theta}_t, \hat{\Sigma}_t)$ be the joint posterior distribution of all action node parameters $\Theta \in \realset^K$, with mean $\hat{\Theta}_t \in \realset^K$ and covariance $\hat{\Sigma}_t \in \realset^{K \times K}$. Let $A_t \in \set{0, 1}^K$ and $A_* \in \set{0, 1}^K$ be indicator vectors of the taken action in round $t$ and the optimal action, respectively. Each action is associated with one leaf node.

Since $\hat{\Theta}_t$ is deterministic given $H_t$, and $A_*$ and $A_t$ are i.i.d.\ given $H_t$, we have
\begin{align*}
  \E{}{A_*\T \Theta_* - A_t\T \Theta_*}
  = \E{}{\condE{A_*\T (\Theta_* - \hat{\Theta}_t)}{H_t}} +
  \E{}{\condE{A_t\T (\hat{\Theta}_t - \Theta_*)}{H_t}}\,.
\end{align*}
Moreover, $\Theta_* - \hat{\Theta}_t$ is a zero-mean random vector independent of $A_t$, and thus $\condE{A_t\T (\hat{\Theta}_t - \Theta_*)}{H_t} = 0$. So we only need to bound the first term above. Let
\begin{align*}
  E_t =
  \set{\forall a \in \cA: |a\T (\Theta_* - \hat{\Theta}_t)|
  \leq \sqrt{2 \log(1 / \delta)} \normw{a}{\hat{\Sigma}_t}}
\end{align*}
be the event that all high-probability confidence intervals hold. Fix history $H_t$. Then by the Cauchy-Schwarz inequality,
\begin{align*}
  \condE{A_*\T (\Theta_* - \hat{\Theta}_t)}{H_t}
  \leq \sqrt{2 \log(1 / \delta)} \, \condE{\normw{A_*}{\hat{\Sigma}_t}}{H_t} +
  \condE{A_*\T (\Theta_* - \hat{\Theta}_t)
  \I{\bar{E}_t}}{H_t}\,.
\end{align*}
Now note that for any action $a$, $a\T (\Theta_* - \hat{\Theta}_t) / \normw{a}{\hat{\Sigma}_t}$ is a standard normal variable. It follows that
\begin{align*}
  \condE{A_*\T (\Theta_* - \hat{\Theta}_t) \I{\bar{E}_t}}{H_t}
  \leq 2 \sum_{a \in \cA} \normw{a}{\hat{\Sigma}_t} \frac{1}{\sqrt{2 \pi}}
  \int_{u = \sqrt{2 \log(1 / \delta)}}^\infty
  u \exp\left[- \frac{u^2}{2}\right] \dif u
  \leq \sqrt{\frac{2}{\pi}} \sigma_{\max} K \delta\,,
\end{align*}
where we use that $\bar{E}_t$ implies $\maxnorm{\hat{\Sigma}^{- \frac{1}{2}}_t (\Theta_* - \hat{\Theta}_t)} \geq \sqrt{2 \log(1 / \delta)}$. Now we combine all inequalities and have
\begin{align*}
  \condE{A_*\T (\Theta_* - \hat{\Theta}_t)}{H_t}
  \leq \sqrt{2 \log(1 / \delta)} \, \condE{\normw{A_t}{\hat{\Sigma}_t}}{H_t} +
  \sqrt{\frac{2}{\pi}} \sigma_{\max} K \delta\,.
\end{align*}
We also used that $A_t$ and $A_*$ are i.i.d.\ given $H_t$.

Since the above bound holds for any history $H_t$, we combine everything and get
\begin{align*}
  \E{}{\sum_{t = 1}^n A_*\T \Theta_* - A_t\T \Theta_*}
  & \leq \sqrt{2 \log(1 / \delta)} \,
  \E{}{\sum_{t = 1}^n \normw{A_t}{\hat{\Sigma}_t}} +
  \sqrt{\frac{2}{\pi}} \sigma_{\max} K n \delta \\
  & \leq \sqrt{2 n\log(1 / \delta)} \,
  \E{}{\sqrt{\sum_{t = 1}^n \normw{A_t}{\hat{\Sigma}_t}^2}} +
  \sqrt{\frac{2}{\pi}} \sigma_{\max} K n \delta \\
  & \leq \sqrt{2 n \log(1 / \delta)}
  \sqrt{\E{}{\sum_{t = 1}^n \normw{A_t}{\hat{\Sigma}_t}^2}} +
  \sqrt{\frac{2}{\pi}} \sigma_{\max} K n \delta\,.
\end{align*}
The second step uses the Cauchy-Schwarz inequality and the third step uses the concavity of the square root.

Since $a$ is an indicator vector, $\normw{A_t}{\hat{\Sigma}_t}^2 = \sigma_{t, A_t}^2$ is the marginal posterior variance of the mean reward of action $A_t$ in round $t$. Likewise, $\sigma_{\max}$ is the maximum marginal prior width of the mean reward of a leaf node. This concludes the proof.

\subsection{Proof of \cref{lem:regret decomposition}}
\label{sec:regret decomposition proof}

Since round $t$ is fixed, we write $L$ instead of $L_t$ and refer to node $\psi_t(i)$ by $i$ for any $i \in [L_t]$.

Fix any node $i$. By the total variance decomposition,
\begin{align*}
  \condvar{\theta_i}{H_t}
  = \condE{\hat{\sigma}_{t, i}^2}{H_t} +
  \condvar{\hat{\theta}_{t, i}}{H_t}\,.
\end{align*}
For Gaussian random variables, $\hat{\sigma}_{t, i}^2$ is independent of $\theta_{i - 1}$, as shown in \eqref{eq:mab_posterior_root}. Therefore,
\begin{align*}
  \condE{\hat{\sigma}_{t, i}^2}{H_t}
  = \hat{\sigma}_{t, i}^2\,.
\end{align*}
For the second term, as shown in \eqref{eq:mab_posterior_root}, $\hat{\theta}_{t, i} = \hat{\sigma}_{t, i}^2 (\sigma_{0, i}^{-2} \theta_{i - 1} + c)$, where $c$ is a constant conditioned on $H_t$. Therefore,
\begin{align*}
  \condvar{\hat{\theta}_{t, i}}{H_t}
  = \frac{\hat{\sigma}_{t, i}^4}{\sigma_{0, i}^4} \condvar{\theta_{i - 1}}{H_t}\,.
\end{align*}
Now we chain all identities for node $i$ and get
\begin{align*}
  \condvar{\theta_i}{H_t}
  = \hat{\sigma}_{t, i}^2 +
  \frac{\hat{\sigma}_{t, i}^4}{\sigma_{0, i}^4} \condvar{\theta_{i - 1}}{H_t}\,.
\end{align*}
Finally, we apply the above identity recursively, from node $L$ all the way up to the root, and get our claim,
\begin{align*}
  \condvar{\theta_h}{H_t}
  & = \hat{\sigma}_{t, L}^2 +
  \frac{\hat{\sigma}_{t, L}^4}{\sigma_{0, L}^4} \condvar{\theta_{L - 1}}{H_t}
  = \hat{\sigma}_{t, L}^2 +
  \frac{\hat{\sigma}_{t, L}^4}{\sigma_{0, L}^4} \hat{\sigma}_{t, L - 1}^2 +
  \frac{\hat{\sigma}_{t, L}^4}{\sigma_{0, L}^4}
  \frac{\hat{\sigma}_{t, L - 1}^4}{\sigma_{0, L - 1}^4} \condvar{\theta_{L - 2}}{H_t} \\
  & = \sum_{i = 1}^L \left(\prod_{j = i + 1}^L
  \frac{\hat{\sigma}_{t, j}^4}{\sigma_{0, j}^4}\right)
  \hat{\sigma}_{t, i}^2\,.
\end{align*}
This completes the proof.

\subsection{Proof of \cref{lem:posterior lower bound}}
\label{sec:posterior lower bound proof}

Since round $t$ is fixed, we write $L$ instead of $L_t$ and refer to node $\psi_t(i)$ by $i$ for any $i \in [L_t]$.

Node $i + 1$ is the only child of node $i$ where the posterior between rounds $t$ and $t + 1$ changes. Thus we have by \eqref{eq:mab_posterior_root} that
\begin{align}
  \hat{\sigma}_{t + 1, i}^{-2} - \hat{\sigma}_{t, i}^{-2}
  = \tilde{\sigma}_{t + 1, i + 1}^{-2} - \tilde{\sigma}_{t, i + 1}^{-2}\,.
  \label{eq:root lower bound}
\end{align}
We decompose $\tilde{\sigma}_{t + 1, i + 1}^{-2} - \tilde{\sigma}_{t, i + 1}^{-2}$ next.

Choose any node $k > i$ except for a leaf, and let node $k + 1$ be its only child where the posterior between rounds $t$ and $t + 1$ changes. Let $a = \sum_{j \in \children(k)} \tilde{\sigma}_{t, j}^{-2}$ and $\varepsilon = \tilde{\sigma}_{t + 1, k + 1}^{-2} - \tilde{\sigma}_{t, k + 1}^{-2}$. We apply the recursive decomposition in \eqref{eq:mab_likelihood} and get
\begin{align}
  \tilde{\sigma}_{t + 1, k}^{-2} - \tilde{\sigma}_{t, k}^{-2}
  & = \frac{1}{\sigma_{0, k}^2 + (a + \varepsilon)^{-1}} -
  \frac{1}{\sigma_{0, k}^2 + a^{-1}}
  = \frac{a^{-1} - (a + \varepsilon)^{-1}}
  {(\sigma_{0, k}^2 + a^{-1}) (\sigma_{0, k}^2 + (a + \varepsilon)^{-1})}
  \nonumber \\
  & = \frac{\varepsilon}{a (a + \varepsilon)}
  \frac{1}{(\sigma_{0, k}^2 + a^{-1}) (\sigma_{0, k}^2 + (a + \varepsilon)^{-1})}
  = \frac{\varepsilon}
  {\sigma_{0, k}^4 (\sigma_{0, k}^{-2} + a) (\sigma_{0, k}^{-2} + a + \varepsilon)}
  \nonumber \\
  & \geq \frac{1}{c} \frac{\hat{\sigma}_{t, k}^4}{\sigma_{0, k}^4} \varepsilon\,,
  \label{eq:node lower bound}
\end{align}
where $c$ is defined in the claim.

Finally, let node $k$ be a leaf node with $n$ observations by round $t$. We apply \eqref{eq:mab_likelihood_leaf} and get
\begin{align}
  \tilde{\sigma}_{t + 1, k}^{-2} - \tilde{\sigma}_{t, k}^{-2}
  & = \frac{1}{\sigma_{0, k}^2 + \sigma^2 / (n + 1)} -
  \frac{1}{\sigma_{0, k}^2 + \sigma^2 / n}
  = \frac{\sigma^2 (n^{-1} - (n + 1)^{-1})}
  {(\sigma_{0, k}^2 + \sigma^2 / n) (\sigma_{0, k}^2 + \sigma^2 / (n + 1))}
  \nonumber \\
  & = \frac{\sigma^2}{n (n + 1)}
  \frac{1}{(\sigma_{0, k}^2 + \sigma^2 / n) (\sigma_{0, k}^2 + \sigma^2 / (n + 1))}
  = \frac{\sigma^{-2}} {\sigma_{0, k}^4
  (\sigma_{0, k}^{-2} + \sigma^{-2} n) (\sigma_{0, k}^{-2} + \sigma^{-2} (n + 1))}
  \nonumber \\
  & \geq \frac{1}{c} \frac{\hat{\sigma}_{t, k}^4}{\sigma_{0, k}^4} \sigma^{-2}\,,
  \label{eq:leaf lower bound}
\end{align}
where $c$ is defined in the claim.

Finally, we apply \eqref{eq:node lower bound} to \eqref{eq:root lower bound}, recursively until the leaf, where we apply \eqref{eq:leaf lower bound}. This completes the proof.

\subsection{Proof of \cref{lem:posterior scaling}}
\label{sec:posterior scaling proof}

Since round $t$ is fixed, we write $L$ instead of $L_t$ and refer to node $\psi_t(i)$ by $i$ for any $i \in [L_t]$.

Let node $i$ be a leaf node with $n$ observations by round $t$. Then
\begin{align*}
  c
  = \frac{\hat{\sigma}_{t + 1, i}^{-2}}{\hat{\sigma}_{t, i}^{-2}}
  = \frac{\sigma_{0, i}^{-2} + \sigma^{-2} (n + 1)}{\sigma_{0, i}^{-2} + \sigma^{-2} n}
  = 1 + \frac{\sigma^{-2}}{\sigma_{0, i}^{-2} + \sigma^{-2} n}
  \leq 1 + \frac{\sigma^{-2}}{\sigma_{0, i}^{-2}}
  = 1 + \frac{\sigma_{0, i}^2}{\sigma^2}\,.
\end{align*}
More generally, for any node $i$, we have by \eqref{eq:mab_posterior_root} that
\begin{align*}
  c
  = \frac{\hat{\sigma}_{t + 1, i}^{-2}}{\hat{\sigma}_{t, i}^{-2}}
  = \frac{\sigma_{0, i}^{-2} + a + \varepsilon}{\sigma_{0, i}^{-2} + a}
  = 1 + \frac{\varepsilon}{\sigma_{0, i}^{-2} + a}
  \leq 1 + \sigma_{0, i}^2 \varepsilon\,,
\end{align*}
where $a = \sum_{j \in \children(i)} \tilde{\sigma}_{t, j}^{-2}$ and $\varepsilon = \tilde{\sigma}_{t + 1, i + 1}^{-2} - \tilde{\sigma}_{t, i + 1}^{-2}$. We bound $\varepsilon$ recursively as in the proof of \cref{lem:posterior lower bound}. The difference is that we need an upper bound now.

For any node $k > i$ except for a leaf, a change in the very last step of \eqref{eq:node lower bound} yields
\begin{align*}
  \tilde{\sigma}_{t + 1, k}^{-2} - \tilde{\sigma}_{t, k}^{-2}
  = \frac{\varepsilon}
  {\sigma_{0, k}^4 (\sigma_{0, k}^{-2} + a) (\sigma_{0, k}^{-2} + a + \varepsilon)}
  \leq \varepsilon
  = \tilde{\sigma}_{t + 1, k + 1}^{-2} - \tilde{\sigma}_{t, k + 1}^{-2}\,.
\end{align*}
Finally, let node $k$ be a leaf node with $n$ observations by round $t$. A change in the very last step of \eqref{eq:leaf lower bound} yields
\begin{align*}
  \tilde{\sigma}_{t + 1, k}^{-2} - \tilde{\sigma}_{t, k}^{-2}
  & = \frac{\sigma^{-2}} {\sigma_{0, k}^4
  (\sigma_{0, k}^{-2} + \sigma^{-2} n) (\sigma_{0, k}^{-2} + \sigma^{-2} (n + 1))}
  \leq \sigma^{-2}\,.
\end{align*}
This implies that
\begin{align*}
  c
  = \frac{\sigma_{t + 1, i}^{-2}}{\sigma_{t, i}^{-2}}
  \leq 1 + \sigma_{0, i}^2 \varepsilon
  \leq 1 + \frac{\sigma_{0, i}^2}{\sigma^2}\,.
\end{align*}
Under the assumption that $\sigma \geq \sigma_{0, i}$ for all nodes $i$, we get $c \leq 2$. This completes the proof.

\subsection{Proof of \cref{thm:regret bound}}
\label{sec:regret bound proof}

First, we apply \cref{lem:bayes regret} and get that
\begin{align*}
  \Bregret(n)
  \leq \sqrt{2 n \cV(n) \log(1 / \delta)} +
  \sqrt{2 / \pi} \sigma_{\max} K n \delta\,,
\end{align*}
where $\cV(n) = \E{}{\sum_{t = 1}^n \sigma_{t, A_t}^2}$ and $\sigma_{t, A_t}^2 = \condvar{\theta_{A_t}}{H_t}$ is the marginal posterior variance in node $A_t$ in round $t$. We prove a worst-case upper bound on $\cV(n)$ below.

We start with a worst-case upper in any round $t$. Since round $t$ is fixed, we write $L$ instead of $L_t$ and refer to node $\psi_t(i)$ by $i$ for any $i \in [L_t]$. Let $\displaystyle s_i = \prod_{j = i + 1}^L \frac{\hat{\sigma}_{t, j}^4}{\sigma_{0, j}^4}$. Then by \cref{lem:regret decomposition},
\begin{align*}
  \hat{\sigma}_{t, A_t}^2
  & = \sigma^2 \frac{\hat{\sigma}_{t, A_t}^2}{\sigma^2}
  = \sigma^2 \sum_{i = 1}^L s_i
  \frac{\hat{\sigma}_{t, i}^2}{\sigma^2}
  \leq \sigma^2 \sum_{i = 1}^L c_i
  \log\left(1 + s_i \frac{\hat{\sigma}_{t, i}^2}{\sigma^2}\right)\,,
\end{align*}
where $c_i$ is a universal upper bound for node $i$ derived as
\begin{align*}
  \frac{s_i \frac{\hat{\sigma}_{t, i}^2}{\sigma^2}}
  {\log\left(1 + s_i \frac{\hat{\sigma}_{t, i}^2}{\sigma^2}\right)}
  \leq \frac{\frac{\hat{\sigma}_{t, i}^2}{\sigma^2}}
  {\log\left(1 + \frac{\hat{\sigma}_{t, i}^2}{\sigma^2}\right)}
  \leq \frac{\sigma_{0, i}^2}
  {\sigma^2 \log\left(1 + \frac{\sigma_{0, i}^2}{\sigma^2}\right)}
  = c_i\,.
\end{align*}
The first inequality holds because $s_i \leq 1$, and $a x / \log(1 + a x) \leq x / \log(1 + x)$ for any $a \in [0, 1]$ and $x > 0$. The second inequality holds because $x / \log(1 + x)$ is maximized when $x$ is, which happens at $\hat{\sigma}_{t, i} = \sigma_{0, i}$.

Now we choose $c \geq 1$ and note that for any node $i \in [L]$,
\begin{align*}
  \log\left(1 + s_i \frac{\hat{\sigma}_{t, i}^2}{\sigma^2}\right)
  = c^{L - i} c^{i - L}
  \log\left(1 + s_i \frac{\hat{\sigma}_{t, i}^2}{\sigma^2}\right)
  \leq c^{L - i}
  \log\left(1 + c^{i - L} s_i \frac{\hat{\sigma}_{t, i}^2}{\sigma^2}\right)\,,
\end{align*}
where the inequality holds because $a \log(1 + x) \leq \log(1 + a x)$ for any $a \in [0, 1]$ and $x > 0$. Moreover,
\begin{align*}
  \log\left(1 + c^{i - L} s_i \frac{\hat{\sigma}_{t, i}^2}{\sigma^2}\right)
  = \log\left(\hat{\sigma}_{t, i}^{-2} + \frac{c^{i - L} s_i}{\sigma^2}\right) -
  \log(\hat{\sigma}_{t, i}^{-2})
  \leq \log(\hat{\sigma}_{t + 1, i}^{-2}) - \log(\hat{\sigma}_{t, i}^{-2})\,,
\end{align*}
where the last step follows from \cref{lem:posterior lower bound}. Now we chain all inequalities, switch to the full notation, and get
\begin{align*}
  \hat{\sigma}_{t, A_t}^2
  \leq \sigma^2 \sum_{i = 1}^{L_t} c^{L_t - i} c_{\psi_t(i)}
  [\log(\hat{\sigma}_{t + 1, \psi_t(i)}^{-2}) - \log(\hat{\sigma}_{t, \psi_t(i)}^{-2})]\,.
\end{align*}
Finally, we sum up the above upper bound over all rounds $t$. Let $h_i$ be the maximum length of any path from node $i$ to its descendant. Since $c \geq 1$ and because of telescoping in the above decomposition, we get
\begin{align*}
  \cV(n)
  \leq \sigma^2 \sum_{i \in \cN} c^{h_i} c_i
  [\log(\hat{\sigma}_{n + 1, i}^{-2}) - \log(\sigma_{0, i}^{-2})]
  = \sigma^2 \sum_{i \in \cN} c^{h_i} c_i
  \log(\sigma_{0, i}^2 \hat{\sigma}_{n + 1, i}^{-2})\,.
\end{align*}
When node $i$ is a leaf, $\hat{\sigma}_{n + 1, i}^{-2} \leq \sigma_{0, i}^{-2} + \sigma^{-2} n$, and thus
\begin{align*}
  \log(\sigma_{0, i}^2 \hat{\sigma}_{n + 1, i}^{-2})
  \leq \log\left(1 + \frac{\sigma_{0, i}^2 n}{\sigma^2}\right)\,.
\end{align*}
When node $i$ is not a leaf, $\hat{\sigma}_{n + 1, i}^{-2} \leq \sigma_{0, i}^{-2} + \sum_{j \in \children(i)} \sigma_{0, j}^{-2}$, and thus
\begin{align*}
  \log(\sigma_{0, i}^2 \hat{\sigma}_{n + 1, i}^{-2})
  \leq \log\left(1 + \sigma_{0, i}^2 \sum_{j \in \children(i)} \sigma_{0, j}^{-2}\right)\,.
\end{align*}
This completes the proof.